\documentclass[conference]{IEEEtran}
\IEEEoverridecommandlockouts

\usepackage{cite}
\usepackage{amsmath,amssymb,amsfonts}
\usepackage{amsthm}
\usepackage{mathtools}
\usepackage{graphicx}
\usepackage{textcomp}
\usepackage{xcolor}
\usepackage{booktabs}
\usepackage{multirow}
\usepackage{array}
\usepackage{tabularx}
\usepackage{microtype}
\usepackage[hyphens]{url}
\usepackage{hyperref}

\ifCLASSOPTIONcompsoc
  \usepackage[caption=false,font=normalsize,labelfont=sf,textfont=sf]{subfig}
\else
  \usepackage[caption=false,font=footnotesize]{subfig}
\fi

\def\BibTeX{{\rm B\kern-.05em{\sc i\kern-.025em b}\kern-.08em
    T\kern-.1667em\lower.7ex\hbox{E}\kern-.125emX}}

\newtheorem{theorem}{Theorem}
\newtheorem{proposition}{Proposition}
\newtheorem{lemma}{Lemma}
\newtheorem{corollary}{Corollary}

\theoremstyle{definition}

\theoremstyle{remark}

\newif\iffull
\ifdefined\buildfull\fulltrue\fi

\begin{document}
\bstctlcite{IEEEexample:BSTcontrol}

\title{Efficient Quantization Training with Stochastic Rounding via Structured Subspace Randomness}

\author{\IEEEauthorblockN{Zeyu Wang, Yiheng Duan, Tao Chen, Donghyun Koo, Sihan Liu, Wenbo Zhang\IEEEauthorrefmark{1}, Chuan He, Liang Liu}
\IEEEauthorblockA{Google, Mountain View, CA, USA\\
\{xavierzeyuwang, yhduan, taoc, koodk, liusih, wenbozhang, hec, liangliu\}@google.com\\
\IEEEauthorrefmark{1}Corresponding author}
}

\maketitle

\begin{abstract}
Low-precision training often relies on stochastic rounding (SR) to keep quantized gradients unbiased and prevent small updates from rounding to zero. Standard SR draws an independent random threshold for every gradient element. However, accelerators generate random numbers on vector units whose throughput lags far behind their matrix units. Drawing one threshold per element saturates the vector unit and can make INT8 training with SR slower than BF16 training. Sharing random thresholds across tensor elements removes this cost but forces those elements to round in the same direction and shifts their rounding errors by a common offset. We find that this offset degrades training primarily through the activation gradient, which propagates the error backward into earlier layers. Because earlier layers cannot cancel an incoming offset using their own local rounding errors, the offset must be removed at the layer where it originates. We propose Structured Subspace Stochastic Rounding (S3R), which quantizes each upstream gradient once for both the activation and weight gradient matrix multiplications and shares a single threshold per token across the hidden dimension. Its error correction, S3R-EC, subtracts the structured mean rounding error from the activation and weight gradients while leaving the optimizer unchanged. In LLM pre-training experiments, uncorrected threshold sharing degrades convergence, whereas S3R-EC eliminates this degradation and matches the validation loss of independent SR. Furthermore, S3R-EC accelerates full INT8 training steps of dense LLMs on TPU v5p by 1.24x to 1.32x over BF16 and by 1.11x over independent SR.
\end{abstract}


\section{Introduction}

Training large language models (LLMs) demands massive compute and memory bandwidth \cite{brown2020language, touvron2023llama, grattafiori2024llama}. To reduce both costs during pre-training and fine-tuning, full quantization training (QT) accelerates the training step itself by executing both the forward- and backward-pass matrix multiplications in 8-bit or 4-bit precision \cite{hubara2017quantized, sun2020ultra, micikevicius2022fp8, peng2023fp8lm, liu2024deepseekv3, xi2024jetfire}. On GPUs, fine-grained FP8 recipes rely on an exponential quantization grid whose step size shrinks near zero, preserving small values under standard deterministic rounding \cite{liu2024deepseekv3}. In contrast, integer formats such as INT8 (the native 8-bit format on accelerators like Google Cloud TPU v5p \cite{xi2024jetfire}) use a uniform grid with a constant step size, and emerging 4-bit floating-point formats (NVFP4 and MXFP4) \cite{nvidia2025nvfp4, tseng2025training} have an extremely coarse grid, making the backward pass in both settings prone to severe underflow.

Specifically, because backward gradients are heavy-tailed and span a wide dynamic range, many of their entries fall below half a quantization step. Deterministic round-to-nearest-even (RTNE) maps these entries to zero every time, causing the affected parameters to stall in a failure mode known as gradient stagnation \cite{gupta2015deep}. Stochastic rounding (SR) prevents this stagnation by rounding each value up or down with probabilities proportional to its distance from the two neighboring grid points. The quantized gradient remains unbiased, and small updates survive in expectation \cite{wang2018training, connolly2021stochastic, croci2022stochastic}.

Standard SR, however, requires an independent uniform random threshold for every gradient element. On modern accelerators, general matrix multiplications (GEMMs) execute on dedicated matrix units (MXUs, such as systolic arrays or Tensor Cores), whereas counter-based pseudo-random number generators (PRNGs, such as Threefry and Philox \cite{salmon2011parallel}) execute on the vector processing unit (VPU). Because matrix throughput has scaled much faster than vector throughput across hardware generations, generating a 32-bit random word per gradient element saturates the VPU and stalls the MXU. Consequently, INT8 training with SR can run slower than BF16 training despite INT8 offering twice the peak MXU throughput of BF16. For example, using JAX's default Threefry generator \cite{bradbury2018jax} on TPU v5p, independent SR executes a single feed-forward network (FFN) block at only 0.58x the speed of BF16 for smaller model dimensions.

To eliminate this random number generation bottleneck, we propose Structured Subspace Stochastic Rounding (S3R), which shares each random threshold across multiple tensor elements. Because each threshold remains uniformly distributed, every element is still rounded without bias. Unlike standard SR with independent draws \cite{connolly2021stochastic}, sharing a threshold forces elements in the same group to round in the same direction, shifting their rounding errors by a common offset. We show that this offset degrades convergence primarily through the activation gradient, which propagates the error backward into earlier layers. Because earlier layers cannot cancel an incoming offset using only their own local rounding errors, we pair S3R with a structured error correction method (S3R-EC) that subtracts the offset directly at the layer where it is generated. Our contributions are as follows.

\begin{enumerate}
    \item We design S3R to eliminate the random number generation overhead of SR. Standard quantized training quantizes the upstream gradient twice in the backward pass, once for the activation gradient and once for the weight gradient, drawing a separate random threshold for every tensor element. S3R instead quantizes each upstream gradient once for both matrix multiplications and draws a single random threshold per token, sharing it across the hidden dimension. Implemented in standard JAX without custom kernels, S3R allows XLA to broadcast these thresholds directly inside vector registers without materializing the full random tensor in memory.
    \item We analyze the error covariance induced by shared thresholds and derive S3R-EC (Sections~\ref{subsec:analysis} and~\ref{subsec:s3r_opt}). Because sharing a threshold across a sequence correlates errors across tokens, S3R shares one threshold per token across the hidden dimension. We prove that per-token sharing preserves the weight gradient variance of independent SR and confines the positive error covariance to each token's activation gradient, which propagates into earlier layers (Lemma~\ref{lemma:shared_offset} and Proposition~\ref{prop:coherent_dx}). We then show that subtracting the mean rounding error removes this covariance (Theorem~\ref{thm:centering} and Corollary~\ref{cor:structured_topologies}). Building on this result, S3R-EC subtracts the mean error along the token and hidden dimensions without extra random draws or optimizer changes (Proposition~\ref{prop:two_way}), while its dX-only variant corrects only the activation gradient to capture nearly the full accuracy gain at lower overhead.
    \item We evaluate convergence in Llama-style language model pre-training and show that while uncorrected threshold sharing degrades validation loss, S3R-EC eliminates this degradation and matches the quality of independent SR (Section~\ref{subsec:convergence}). In end-to-end benchmarks on TPU v5p, S3R-EC accelerates full training steps of three dense LLMs \cite{grattafiori2024llama, qwen2025qwen3, gemma2025gemma3} by 1.24x to 1.32x over BF16 and by 1.11x over independent SR (Section~\ref{subsec:results}). On two mixture-of-experts (MoE) models \cite{qwen2025qwen3, liu2024deepseek}, S3R and S3R-EC achieve speedups of 1.09x to 1.11x and 1.04x to 1.06x over independent SR, respectively.
\end{enumerate}

\begin{figure*}[t!]
\centering
\includegraphics[width=0.94\textwidth]{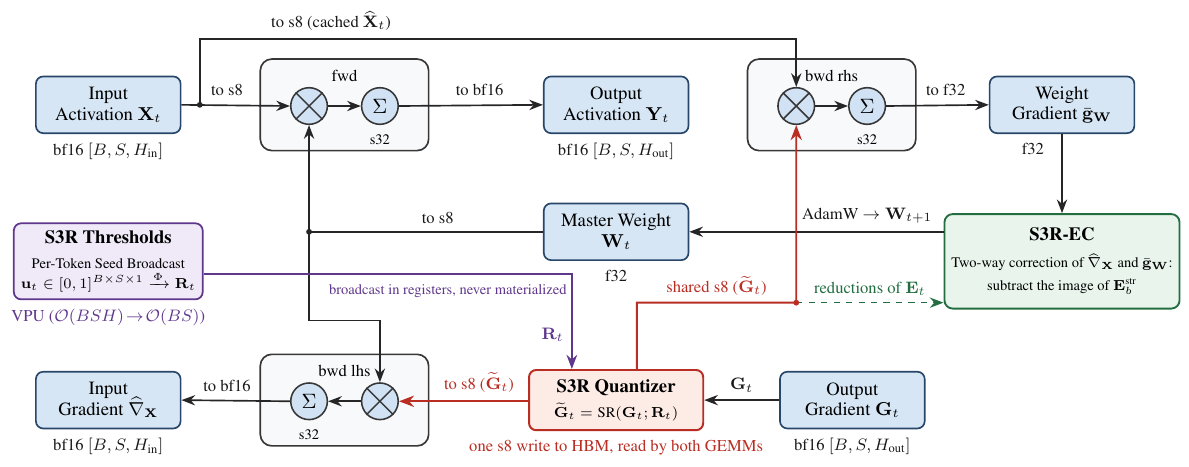}
\vspace{-4pt}
\caption{Overview of INT8 (S8) quantization training with S3R and its structured error correction S3R-EC. On the VPU, S3R broadcasts a compact seed $\mathbf{u}_t \in [0, 1]^{B \times S \times 1}$, one threshold per token, to $\mathbf{R}_t = \Phi(\mathbf{u}_t)$ inside the vector registers, so $\mathbf{R}_t$ is never materialized. In the backward pass, the upstream gradient $\mathbf{G}_t$ is quantized \emph{once}, and the int8 result is written to HBM once and read by both the activation-gradient product $\widetilde{\mathbf{G}}_t \mathbf{W}_t^\top$ and the weight-gradient product $\widehat{\mathbf{X}}_t^\top \widetilde{\mathbf{G}}_t$. S3R-EC subtracts the token-constant and channel-constant parts of the rounding error from $\widehat{\nabla}_{\mathbf{X}}$ and from the FP32 weight gradient $\bar{\mathbf{g}}_{\mathbf{W}}$, and the AdamW update is unchanged.}
\label{fig:int8_quantization}
\vspace{-6pt}
\end{figure*}

\section{Related Work}
\label{sec:related}

\subsection{Low-Precision Training and Stochastic Rounding}
Full quantization training runs the forward product and both backward products in low precision \cite{hubara2017quantized, sun2020ultra, micikevicius2022fp8}. FP8 recipes train LLMs at scale with fine-grained scaling and stability fixes \cite{peng2023fp8lm, liu2024deepseekv3, fishman2025scaling} and also compress optimizer states and activations \cite{xi2024coat}. INT8 recipes quantize the whole data flow \cite{xi2024jetfire} or keep the weight-gradient product in 16 bits \cite{wortsman2023stable}. At 4 bits, NVFP4 and MXFP4 recipes apply SR to gradients after a random Hadamard transform \cite{nvidia2025nvfp4, tseng2025training}, and LUQ rounds them without bias on a logarithmic grid \cite{chmiel2023accurate}. Other 4-bit recipes sample gradients by leverage score \cite{xi2023training} or, like Quartet, prefer round-to-nearest \cite{castro2025quartet}.

By preventing gradient stagnation \cite{gupta2015deep}, unbiased SR enabled early 8-bit floating-point training \cite{wang2018training}. Under independent draws, its errors are mean-independent, which gives probabilistic error bounds that grow with the square root of the problem size \cite{connolly2021stochastic}. SR also implicitly regularizes tall-and-thin matrices \cite{dexter2024stochastic}, and Croci et al.\ survey its uses \cite{croci2022stochastic}. In training, SR yields unbiased gradient estimators whose variance depends on the quantizer \cite{chen2020statistical}, and its interaction with Adam has been analyzed for LLMs \cite{ozkara2025stochastic}. These analyses assume independent or mean-independent errors, which shared thresholds violate (Lemma~\ref{lemma:shared_offset}, Proposition~\ref{prop:coherent_dx}).

\subsection{The Cost of Randomness}
Counter-based generators such as Philox and Threefry compute each random word as a keyed function of a counter \cite{salmon2011parallel}, and JAX uses Threefry by default \cite{bradbury2018jax}. Each word takes many rounds of integer operations on the vector unit, which has scaled far more slowly than the matrix units. In inference, a few extra vector instructions per multiply-accumulate can already make a kernel vector-bound \cite{lin2024qserve}. On TPU v5p, the independent-SR quantizer issues more than a hundred vector operations per element (Section~\ref{subsec:results}). Materializing random tensors also adds off-chip traffic, which costs far more energy than arithmetic \cite{han2015learning, han2016eie}. S3R broadcasts its thresholds in vector registers without materializing them, and it writes the quantized INT8 gradient to HBM once for both backward products.

Hardware support lowers the cost of SR. Graphcore's IPU supports SR in binary16 and binary32 arithmetic \cite{croci2022stochastic}, proposed SR units accelerate SR in fixed-point arithmetic \cite{mikaitis2021stochastic}, and Blackwell GPUs provide conversion instructions (\texttt{cvt.rs}) that round to FP16, BF16, FP8, FP6 or FP4 with random bits supplied by software \cite{nvidia2025ptx}. Such support, however, is tied to particular chips and number formats, and it still needs a random value for every element, which software must generate for \texttt{cvt.rs}. S3R needs no special hardware and far fewer random values.

With few random bits per element, SR becomes biased unless the implementation compensates. El Arar et al.\ analyze limited-precision SR \cite{elarar2025probabilistic}, and Fitzgibbon and Felix show that natural few-bit implementations can be significantly biased \cite{fitzgibbon2025stochastic}. S3R uses 7-bit thresholds shifted by half a grid step (Section~\ref{sec:method}) and thus falls in this regime. Slicing one random word into several narrow values, as S3R also does, is existing practice in JAX's legacy Threefry path \cite{bradbury2018jax} and in \texttt{cvt.rs} \cite{nvidia2025ptx}. Generator-free SR instead takes random bits from values already present in training \cite{yuan2022you}. Our contribution concerns the far larger saving from sharing one random value across many elements, and what this sharing does to training.

\subsection{Shared Randomness and Error Correction}
SR is equivalent to round-to-nearest with non-subtractive dither, which adds a uniform offset over one quantization step before rounding \cite{gray1993dithered}. Prior work uses non-subtractive dither to sparsify backward gradients \cite{wiedemann2020dithered} and couples random draws across distributed clients so their quantization errors partially cancel upon aggregation \cite{suresh2022correlated}. Rather than coupling draws across clients to cancel errors, S3R shares random thresholds within a tensor to reduce generation cost. Because LLM token representations are anisotropic with persistent outlier channels across tokens \cite{ethayarajh2019contextual, xiao2023smoothquant, lin2024awq}, S3R shares each threshold across the hidden dimension within a token rather than across tokens. Within each token, this shared threshold shifts element rounding errors by a common offset (Proposition~\ref{prop:coherent_dx}).

To reduce quantization error after rounding, classical subtractive dither subtracts the exact dither signal from every element \cite{gray1993dithered}, while post-training bias correction subtracts the expected output error of quantized weights \cite{nagel2019data}. S3R-EC instead operates on the realized backward gradient error without per-element subtraction, removing only the mean rounding error along the token and hidden dimensions to eliminate the shared offset (Section~\ref{subsec:s3r_opt}). Because these mean errors have zero expectation, subtracting their linear effect acts as a control variate that keeps the gradients unbiased while complementing few-bit SR techniques.

\section{Methodology and Theoretical Framework}
\label{sec:method}
\label{methodology}

In this section, we first introduce SR in quantized backpropagation, then describe how S3R works and shares thresholds, and analyze the effect of sharing, which motivates the structured error correction S3R-EC for both the dX-only and two-way variants.

\begin{figure*}[t]
\centering
\includegraphics[width=\textwidth]{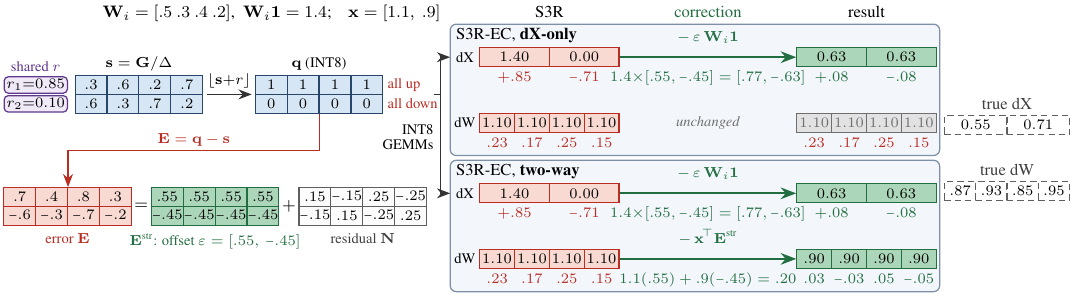}
\vspace{-12pt}
\caption{Worked example of S3R-EC (two tokens, four channels, units of $\Delta$). Left: one shared threshold per token ($r_1 = 0.85$, $r_2 = 0.10$) rounds all channels of token~1 up and all of token~2 down; the known error $\mathbf{E} = \mathbf{q} - \mathbf{s}$ splits into $\mathbf{E}^{\text{str}}$ and a residual $\mathbf{N}$ with zero row and column sums. Here $\mathbf{E}^{\text{str}}$ reduces to one offset $\varepsilon$ per token (in general it also has a per-channel term). Right: both variants subtract $\varepsilon\,\mathbf{W}_i\mathbf{1}$ from the activation gradient (row $\mathbf{W}_i$ of $\mathbf{W}$); two-way also subtracts $\mathbf{x}^\top\mathbf{E}^{\text{str}}$ from the weight gradient (input feature $\mathbf{x}$), which dX-only leaves unchanged.}
\label{fig:ec_example}
\vspace{-6pt}
\end{figure*}

\subsection{Quantized Backpropagation with Stochastic Rounding}

Consider a linear projection $\mathbf{Y} = \mathbf{X}\mathbf{W}$ of a Transformer with activations $\mathbf{X} \in \mathbb{R}^{B \times S \times H_{\text{in}}}$, weights $\mathbf{W} \in \mathbb{R}^{H_{\text{in}} \times H_{\text{out}}}$, micro-batch size $B$, and sequence length $S$. Writing $H = H_{\text{out}}$ for brevity, let $\mathbf{G} = \partial\ell/\partial\mathbf{Y} \in \mathbb{R}^{B \times S \times H}$ be the upstream gradient of the summed per-token loss $\ell = \sum_{b,s}\ell_{b,s}$, where the training loss is the token mean $\mathcal{L} = \ell/(BS)$. The backward pass computes two matrix products: the activation-gradient (dX) product $\mathbf{G}\mathbf{W}^\top$ and the weight-gradient (dW) product $\mathbf{X}^\top\mathbf{G}$, where $\mathbf{G}$ serves as the left-hand (LHS) and right-hand (RHS) operand, respectively.

In the INT8 pipeline (Fig.~\ref{fig:int8_quantization}), each operand of the forward and backward products is scaled, rounded from BF16 to S8, multiplied as an S8 general matrix multiplication (GEMM) with S32 accumulation, and rescaled. In the backward pass, SR is applied to $\mathbf{G}$. For a step size $\Delta > 0$ and thresholds $\mathbf{R} \in [0, 1]^{B \times S \times H}$, the stochastic rounding operator $\widetilde{\mathbf{G}} = \text{SR}(\mathbf{G}; \mathbf{R})$ rounds each element up when its normalized residual $\mathbf{P} = \mathbf{G}/\Delta - \lfloor \mathbf{G}/\Delta \rfloor$ reaches its threshold:
\begin{equation}
\label{eq:sr_def}
\text{SR}(G_{b,s,h}; R_{b,s,h}) = \Delta \left( \left\lfloor \frac{G_{b,s,h}}{\Delta} \right\rfloor + \mathbb{I}\left(P_{b,s,h} \ge R_{b,s,h}\right) \right),
\end{equation}
where $\mathbb{I}(\cdot)$ is the indicator function. With rounding error $\mathbf{E}(\mathbf{G}, \mathbf{R}) = \widetilde{\mathbf{G}} - \mathbf{G}$ and uniform thresholds, each entry $E_{b,s,h} = \Delta (\mathbb{I}(P_{b,s,h} \ge R_{b,s,h}) - P_{b,s,h})$ is a centered Bernoulli variable scaled by $\Delta$. For example, with $\Delta = 1$, the value $0.3$ is rounded to $1$ with probability $0.3$ and to $0$ otherwise, so its mean is exactly $0.3$. Our quantizer computes $\Delta\lfloor G/\Delta + R\rfloor$, which is (\ref{eq:sr_def}) with the threshold $1 - R$; this threshold is uniform and shared wherever $R$ is, so all results apply.

The forward product applies RTNE with per-token scales $\mathbf{s}_x$ and per-output-channel scales $\mathbf{s}_w$, so that $\mathbf{X} \approx \mathrm{diag}(\mathbf{s}_x)\widehat{\mathbf{X}}$ and $\mathbf{W} \approx \widehat{\mathbf{W}}\,\mathrm{diag}(\mathbf{s}_w)$ with S8 operands $\widehat{\mathbf{X}}$ and $\widehat{\mathbf{W}}$. When both backward products share a quantized gradient (as in RTNE and S3R), the backward pass folds both scales into the gradient and quantizes the result once:

\begin{equation}
\label{eq:scale_fold}
\begin{aligned}
\mathbf{G}' &= \mathrm{diag}(\mathbf{s}_x)\,\mathbf{G}\,\mathrm{diag}(\mathbf{s}_w),\\
\widehat{\nabla}_{\mathbf{X}} &= \mathrm{diag}(\mathbf{s}_x)^{-1}\,\Delta\mathbf{Q}\,\widehat{\mathbf{W}}^\top,\\
\widehat{\nabla}_{\mathbf{W}} &= \Delta\,\widehat{\mathbf{X}}^\top\mathbf{Q}\,\mathrm{diag}(\mathbf{s}_w)^{-1},
\end{aligned}
\end{equation}
where $\Delta = \max|\mathbf{G}'|/127$ and $\mathbf{Q} = \text{SR}(\mathbf{G}'; \mathbf{R})/\Delta$ holds the S8 codes. With $\mathbf{G}'$ in place of $\Delta\mathbf{Q}$, the two products equal $\mathbf{G}\mathbf{W}^\top$ and $\mathbf{X}^\top\mathbf{G}$ for the dequantized operands. One S8 tensor feeds both GEMMs, which contract its channels (dX) and its tokens (dW), so only a per-tensor step can be applied after S32 accumulation; fine-grained recipes such as DeepSeek-V3 \cite{liu2024deepseekv3} instead requantize an operand for its transposed use. Below, $\mathbf{G}$, $\mathbf{W}$ and $\mathbf{X}$ denote $\mathbf{G}'$, $\widehat{\mathbf{W}}$ and $\widehat{\mathbf{X}}$, and the fixed rescalings by $1/\mathbf{s}_x$ and $1/\mathbf{s}_w$, which preserve unbiasedness, are left implicit.

Conventional QT implementations draw independent full-size thresholds $\mathbf{R}^{\text{LHS}}, \mathbf{R}^{\text{RHS}} \sim \mathcal{U}([0, 1]^{B \times S \times H})$ for the two products: $\widehat{\nabla}_{\mathbf{X}} = \text{SR}(\mathbf{G}; \mathbf{R}^{\text{LHS}})\mathbf{W}^\top$ estimates $\mathbf{G}\mathbf{W}^\top$, and $\widehat{\nabla}_{\mathbf{W}} = \mathbf{X}^\top \text{SR}(\mathbf{G}; \mathbf{R}^{\text{RHS}})$ estimates $\mathbf{X}^\top\mathbf{G}$, so $\bar{\mathbf{g}}_{\mathbf{W}} = \widehat{\nabla}_{\mathbf{W}}/(BS)$ estimates $\nabla_{\mathbf{W}}\mathcal{L}$. This costs $2BSH$ 32-bit Threefry words per projection, one per element and product. Each word is hashed on the VPU, so the quantizer issues more than ten times as many vector operations as an RTNE quantizer, while the S8 GEMMs on the MXU take the same time.

\subsection{Structured Subspace Randomness}
S3R reduces this cost in two steps (Fig.~\ref{fig:int8_quantization}): it uses one threshold tensor for both backward products, and it broadcasts a small seed to the full tensor. Only the second step changes the error distribution of each product. Each threshold is a 7-bit \texttt{bf16} value; as in existing implementations \cite{bradbury2018jax, nvidia2025ptx}, S3R cuts four thresholds from each 32-bit Threefry word \cite{salmon2011parallel}, and a half-step shift of the threshold grid bounds the resulting rounding bias by $\Delta/256$. Lemma~\ref{lemma:unbiasedness} thus holds up to this discretization. S3R builds on partitionable Threefry, which gives bitwise-identical thresholds under pure-FSDP, mixed and pure-TP factorizations without added collectives.

\subsubsection{Inter-Operand Sharing}
With separate PRNG keys, the gradient is quantized twice. S3R uses one threshold tensor for both products, $\mathbf{R}^{\text{LHS}} = \mathbf{R}^{\text{RHS}} = \mathbf{R}$, so the quantizations coincide, as under RTNE, and $\mathbf{Q}$ in (\ref{eq:scale_fold}) is computed once. XLA compiles the quantizer into a fusion that writes $\mathbf{Q}$ to HBM once, at one byte per element, for both GEMMs to read. This halves the backward quantization work and the S8 writes ($\alpha_{\text{share}} = 1/2$), and both products remain unbiased (Lemma~\ref{lemma:unbiasedness}, Part~2). Their errors become correlated, but the correlation enters none of the second moments analyzed below.

\subsubsection{Threshold Broadcasting}
\label{subsec:broadcast}
The second step draws a compact seed $\mathbf{u} \sim \mathcal{U}([0, 1]^{d_0})$ with $d_0 \ll BSH$ and expands it to $\mathbf{R} = \Phi(\mathbf{u}) \in [0, 1]^{B \times S \times H}$ with a broadcast $\Phi$. XLA fuses the broadcast into the quantizer, so thresholds are broadcast in vector registers and $\mathbf{R}$ is never materialized. Every $\Phi$ keeps thresholds marginally uniform; sharing more saves draws but correlates more errors, because elements that share a threshold tend to round in the same direction. Independent SR shares nothing; we study two choices:
\begin{itemize}
    \item \textbf{Per token, $\Phi_{\text{2D}}$ (default).} Here $\mathbf{u} \in [0, 1]^{B \times S \times 1}$ and $R_{b,s,h} = u_{b,s}$. This cuts PRNG work by a factor of $H$ and keeps rounding errors independent across tokens. Errors within a token remain positively correlated across channels (Corollary~\ref{cor:structured_topologies}); S3R-EC removes this covariance (Proposition~\ref{prop:two_way}).
    \item \textbf{Per example, $\Phi_{\text{bcast}}$ (ablation).} Here $\mathbf{u} \in [0, 1]^{B \times 1 \times 1}$ and $R_{b,s,h} = u_b$. This needs only $B$ draws per layer, but it correlates the errors of all tokens of an example, and their common part reaches the activation gradient of every token (Proposition~\ref{prop:coherent_dx}), so we use it only as an ablation.
\end{itemize}

\subsection{Why Sharing Hurts and How S3R-EC Repairs It}
\label{subsec:analysis}
Sharing a threshold biases no gradient, but it ties together the rounding errors of the elements that share it. We show in turn that these errors are positively correlated, where the correlation adds error, and that subtracting their known mean, which is what S3R-EC does, removes the correlation. Throughout, $\mathbf{G}$ is fixed and only the thresholds are random; we assume nothing about the residuals and require only that every threshold be uniform on $[0, 1)$.

\subsubsection{First-Order Unbiasedness}
SR is unbiased whenever each threshold is marginally uniform. Expectation is linear, so correlations among the thresholds do not matter.

\begin{lemma}[First-Order Marginal Unbiasedness]
\label{lemma:unbiasedness}
Fix $\mathbf{G} \in \mathbb{R}^{B \times S \times H}$ (fixed residuals), and let $\mathbf{R} = \Phi(\mathbf{u})$ be any random threshold field whose entries are marginally uniform, $R_{b,s,h} \sim \mathcal{U}(0, 1)$. Then:
\begin{enumerate}
    \item \textbf{Linear maps:} $\mathbb{E}_{\mathbf{u}}[T(\text{SR}(\mathbf{G}; \Phi(\mathbf{u})))] = T(\mathbf{G})$ for every linear operator $T$.
    \item \textbf{Shared quantized gradient:} When both backward products use the same $\widetilde{\mathbf{G}}$, $\mathbb{E}_{\mathbf{u}}[\widetilde{\mathbf{G}}\mathbf{W}^\top] = \mathbf{G}\mathbf{W}^\top$ and $\mathbb{E}_{\mathbf{u}}[\mathbf{X}^\top\widetilde{\mathbf{G}}] = \mathbf{X}^\top\mathbf{G}$. Hence $\mathbb{E}_{\mathbf{u}}[\bar{\mathbf{g}}_{\mathbf{W}}] = \nabla_{\mathbf{W}}\mathcal{L}$.
\end{enumerate}
\end{lemma}
The lemma composes across layers: backpropagation is linear in the upstream gradient, and each layer draws its thresholds independently of the noise from above, so by the tower property every layer's gradient is unbiased with respect to the backward pass of the executed forward graph (with $\widehat{\mathbf{X}}, \widehat{\mathbf{W}}$ if it is quantized). The topology $\Phi$ can therefore be chosen for its cost and its second moments alone.

\subsubsection{Broadcasting Makes the Errors Move Together}
With independent thresholds, every element chooses its rounding direction on its own, and the errors of different elements are independent. With a broadcast threshold, one draw chooses for the whole group: a large draw rounds most of its elements up, and a small draw rounds most of them down.

\begin{lemma}[Covariance of Shared-Threshold Errors]
\label{lemma:shared_offset}
Round with $\Delta\lfloor G/\Delta + R\rfloor$ and marginally uniform thresholds, and let $P_i$ be the residual of element $i$. If elements $i$ and $k$ share one threshold, their rounding errors satisfy
\begin{equation}
\label{eq:shared_offset}
\text{Cov}(E_i, E_k) = \Delta^2\big(\min(P_i, P_k) - P_iP_k\big) \in [0, \Delta^2/4].
\end{equation}
If their thresholds are independent, $\text{Cov}(E_i, E_k) = 0$.
\end{lemma}
Every pair of errors in a group that shares a threshold is thus nonnegatively correlated. The covariance is largest when both residuals are near $\frac{1}{2}$ and vanishes when either element is almost exactly representable. If the residuals of a group are spread evenly over $[0, 1)$, each error has variance about $\Delta^2/6$ and each pair has covariance about $\Delta^2/12$, so half of the variance of every element is shared with the rest of its group. Under $\Phi_{\text{2D}}$, a group is the $D = H_{\text{out}}$ channels of one token, and different tokens have independent thresholds.

\subsubsection{Where the Correlation Adds Error}
A backward product adds up errors with weights, and the variance of such a sum is the sum of the variances plus the sum of all pairwise covariances. Independent SR pays only the first part. Correlated errors also pay the second, because they no longer cancel.

\begin{proposition}[Where Sharing Adds Error]
\label{prop:coherent_dx}
Use per-token thresholds, $\Phi_{\text{2D}}$, and let $\mathbf{E}_{b,s} \in \mathbb{R}^D$ be the rounding error of token $(b, s)$.
\begin{enumerate}
    \item Every entry of the weight-gradient error has the same variance as under independent SR.
    \item Entry $i$ of the activation-gradient error of token $(b, s)$, $\sum_j W_{ij}E_{b,s,j}$, has the variance it would have under independent SR plus $\sum_{j \neq k} W_{ij}W_{ik}\,\text{Cov}(E_{b,s,j}, E_{b,s,k})$. If every pair of channels of the token has covariance $c$, this extra term is $c\big((\mathbf{W}_i\mathbf{1}_D)^2 - \|\mathbf{W}_i\|_2^2\big)$.
    \item This error is part of the upstream gradient of the layer below. A correction of this layer's weight gradient leaves it unchanged, and the layer below, whose thresholds are independent of it, cannot cancel it with a correction built from its own rounding error.
\end{enumerate}
\end{proposition}
Per-token sharing thus keeps the weight gradient exactly as accurate as independent SR and puts the whole cost into the activation gradient, which sums over exactly the channels that share a threshold. The extra term is large when the rows of $\mathbf{W}$ have a common component, because $(\mathbf{W}_i\mathbf{1}_D)^2$ then grows with the square of the width while $\|\mathbf{W}_i\|_2^2$ grows linearly. The shared part of a token's error is its mean $\varepsilon_{b,s} = \frac{1}{D}\mathbf{1}_D^\top\mathbf{E}_{b,s}$, which enters the activation gradient as $\varepsilon_{b,s}\mathbf{W}\mathbf{1}_D$, along the same direction for every token. By the last statement, this error has to be removed in the activation gradient, where it arises. In a Transformer it does not arrive unchanged: the RMSNorm backward rescales each token and removes its component along the normalized input, and the attention backward mixes tokens. The proposition is therefore a mechanism, not an exact model. If all tokens of an example shared one threshold, the weight gradient would pay the same price over the tokens.

\subsubsection{Subtracting the Mean Removes the Correlation}
The quantizer knows every rounding error exactly, so it can compute the mean error of a group and subtract it. The theorem below shows what this does to correlated errors; it holds for any covariances.

\begin{theorem}[Centering Reduces Covariance]
\label{thm:centering}
Let $X_1, \dots, X_n$, $n \ge 2$, be zero-mean errors with average variance $\bar{\sigma}^2$ and average pairwise covariance $\bar{c}$, let $\bar{X}$ be their mean, and let $Y_i = X_i - \bar{X}$. Then $\text{Var}(\bar{X}) = \frac{1}{n}\bar{\sigma}^2 + \frac{n-1}{n}\bar{c}$, and:
\begin{enumerate}
    \item \textbf{Covariance:} The average pairwise covariance of the $Y_i$ is $\bar{c} - \text{Var}(\bar{X}) = -\frac{1}{n}(\bar{\sigma}^2 - \bar{c}) \le 0$.
    \item \textbf{Variance:} The average variance of the $Y_i$ is $\bar{\sigma}^2 - \text{Var}(\bar{X}) = \frac{n-1}{n}(\bar{\sigma}^2 - \bar{c})$.
    \item \textbf{Equal covariances:} If every $X_i$ has variance $\sigma^2$ and every pair has covariance $c$, each pair's covariance falls from $c$ to $-\frac{1}{n}(\sigma^2 - c)$, and for any weights $w_i$ with mean $\bar{w}$, $\text{Var}(\sum_i w_iY_i) = (\sigma^2 - c)\sum_i(w_i - \bar{w})^2$.
\end{enumerate}
\end{theorem}
The average covariance and the average variance fall by the same amount, the variance of the mean. For errors that share a threshold, $\bar{c} \ge 0$ by Lemma~\ref{lemma:shared_offset}, and for large $n$ this amount is about $\bar{c}$. Subtracting the mean therefore removes the positive correlation that broadcasting created, and the centered errors are even slightly negatively correlated, like draws without replacement. The case of equal covariances shows the mechanism. Each error is then a shared part with variance $c$ plus a private part with variance $\sigma^2 - c$. Subtracting the mean deletes the shared part and keeps the private parts, which are uncorrelated, so any weighted sum of centered errors has at most the variance of the same sum of independent errors with variance $\sigma^2 - c$. For evenly spread residuals, this is about $\Delta^2/12$, half of the $\Delta^2/6$ of independent SR.

The first two statements are averages and hold for any covariances. The third needs equal covariances, which fixed residuals do not give: by Lemma~\ref{lemma:shared_offset}, an element that is almost exactly representable is nearly uncorrelated with the rest of its group, and subtracting the group mean adds the noise of the group to it. In general, element $i$ becomes less noisy exactly when $2\,\text{Cov}(X_i, \bar{X}) > \text{Var}(\bar{X})$. For $X_1 = -Z$ and $X_2 = X_3 = Z$, the variance of the first element grows to $\frac{16}{9}\text{Var}(Z)$, while the average variance falls to $\frac{8}{9}\text{Var}(Z)$. Single outputs of a GEMM can behave in the same way. In an exact computation on a synthetic layer with unequal residuals, 17--27\% of the activation-gradient outputs become noisier after centering, while their summed variance falls by 40--63\% and stays below that of independent SR.

\subsubsection{Which Axis to Center}
Centering removes covariance only along an axis whose errors are correlated.

\begin{corollary}[Centering under Per-Token Thresholds]
\label{cor:structured_topologies}
Use $\Phi_{\text{2D}}$ and consider one example.
\begin{enumerate}
    \item Subtracting the mean error of a token over its $D$ channels lowers their average pairwise covariance from $\bar{c}_{b,s} \ge 0$ to $-\frac{1}{D}(\bar{\sigma}^2_{b,s} - \bar{c}_{b,s})$, where $\bar{\sigma}^2_{b,s}$ is their average variance.
    \item Subtracting the mean error of channel $j$ over the $S$ tokens finds no covariance to remove, because the tokens have independent thresholds, and lowers the average variance $\bar{\sigma}^2_{b,j}$ of its entries by only $\frac{1}{S}\bar{\sigma}^2_{b,j}$.
\end{enumerate}
\end{corollary}
Under per-token thresholds, the activation-gradient term of S3R-EC, which centers each token, therefore does nearly all of the work, and the weight-gradient term adds little. If the tokens of an example also shared a threshold, centering over tokens would remove their covariance as well, and both terms would matter. Section~\ref{subsec:convergence} confirms both predictions.

\subsection{Structured Error Correction (S3R-EC)}
\label{subsec:s3r_opt}
By Proposition~\ref{prop:coherent_dx}, the cost of sharing is the covariance of the errors within each token, which enters the activation gradient and travels to every lower layer. S3R-EC removes this covariance by subtracting the known mean error, the operation analyzed in Theorem~\ref{thm:centering}, along both axes of the error of each example.

\subsubsection{Two-Way Correction}
The idea is simple (Fig.~\ref{fig:ec_example}). The quantizer knows $\mathbf{G}$ and $\widetilde{\mathbf{G}}$, and hence the rounding error $\mathbf{E} = \widetilde{\mathbf{G}} - \mathbf{G}$, exactly. S3R-EC takes two averages of $\mathbf{E}$: the mean error of each token over its channels, which carries the covariance created by the shared threshold, and the mean error of each channel over the tokens of an example. After the unchanged INT8 GEMMs, it subtracts their effect: from the activation gradient, the averages times $\mathbf{W}$, which for the token mean is the mean times the precomputed row sums $\mathbf{W}\mathbf{1}_D$; from the weight gradient, the inputs times the same averages. What remains sums to zero along every row and every column. No extra random numbers are needed.

Formally, with $\mathbf{G}$, $\mathbf{X}$ and $\mathbf{W}$ as in~(\ref{eq:scale_fold}), the averages of example $b$ form the orthogonal projection of $\mathbf{E}_b \in \mathbb{R}^{S \times D}$ ($D = H_{\text{out}}$) onto the subspace $\mathcal{S} = \{ \mathbf{1}_S\mathbf{a}^\top + \mathbf{c}\mathbf{1}_D^\top \}$ of sums of row-constant and column-constant matrices:
\begin{equation}
\label{eq:estr}
\begin{gathered}
\bar{\mathbf{e}}_b = \tfrac{1}{S}\mathbf{E}_b^\top\mathbf{1}_S, \quad \boldsymbol{\varepsilon}_b = \tfrac{1}{D}\mathbf{E}_b\mathbf{1}_D, \quad m_b = \tfrac{1}{SD}\mathbf{1}_S^\top\mathbf{E}_b\mathbf{1}_D, \\
\mathbf{E}^{\text{str}}_b = \mathbf{1}_S\bar{\mathbf{e}}_b^\top + (\boldsymbol{\varepsilon}_b - m_b\mathbf{1}_S)\mathbf{1}_D^\top .
\end{gathered}
\end{equation}
Passing $\mathbf{E}^{\text{str}}_b$ through both backward products and subtracting the result gives
\begin{align}
\label{eq:s3r_opt_dx}
\widehat{\nabla}^{\text{corr}}_{\mathbf{X}_b} &= \widehat{\nabla}_{\mathbf{X}_b} - \mathbf{1}_S(\mathbf{W}\bar{\mathbf{e}}_b)^\top - (\boldsymbol{\varepsilon}_b - m_b\mathbf{1}_S)(\mathbf{W}\mathbf{1}_D)^\top, \\
\label{eq:s3r_opt_gc}
\bar{\mathbf{g}}_{\mathbf{W}}^{\text{corr}} &= \bar{\mathbf{g}}_{\mathbf{W}} - \frac{1}{BS}\sum_{b=1}^{B}\Big[\bar{\mathbf{x}}_b\mathbf{e}_b^\top + \mathbf{X}_b^\top(\boldsymbol{\varepsilon}_b - m_b\mathbf{1}_S)\mathbf{1}_D^\top\Big],
\end{align}
where $\bar{\mathbf{x}}_b = \frac{1}{S}\mathbf{X}_b^\top\mathbf{1}_S$ and $\mathbf{e}_b = S\bar{\mathbf{e}}_b$. In (\ref{eq:s3r_opt_gc}), the first term removes the mean error of each channel over the tokens of an example, and the second removes the mean error that each token shares across its channels. The corrected gradients go to the unchanged optimizer, AdamW~\cite{loshchilov2019decoupled} in our experiments. The correction keeps both gradients unbiased and leaves the remaining errors of every token and every channel with a nonpositive average covariance:

\begin{proposition}[Properties of the Two-Way Correction]
\label{prop:two_way}
Consider any threshold topology with marginally uniform thresholds, and fix the residuals. Let $\mathbf{N}_b = \mathbf{E}_b - \mathbf{E}^{\text{str}}_b$ be the error that remains. Then:
\begin{enumerate}
    \item \textbf{Unbiasedness:} $\mathbb{E}_{\mathbf{u}}[\widehat{\nabla}^{\text{corr}}_{\mathbf{X}}] = \mathbf{G}\mathbf{W}^\top$ and $\mathbb{E}_{\mathbf{u}}[\bar{\mathbf{g}}^{\text{corr}}_{\mathbf{W}}] = \nabla_{\mathbf{W}}\mathcal{L}$.
    \item \textbf{Orthogonal Decomposition:} $\mathbf{E}^{\text{str}}_b$ is the orthogonal projection of $\mathbf{E}_b$ onto $\mathcal{S}$. Hence $\mathbf{N}_b^\top\mathbf{1}_S = \mathbf{0}$, $\mathbf{N}_b\mathbf{1}_D = \mathbf{0}$, and $\|\mathbf{N}_b\|_F \le \|\mathbf{E}_b\|_F$ on every realization.
    \item \textbf{Covariance Removal:} Within every token and within every channel of example $b$, the remaining errors sum to zero, so their average pairwise covariance is $-\frac{1}{n-1}$ times their average variance, with $n = D$ for a token and $n = S$ for a channel. In particular, the shared term $\varepsilon_{b,s}\mathbf{W}\mathbf{1}_D$ of every token's activation-gradient error is removed.
    \item \textbf{Variance Reduction:} $\sum_{s,j}\text{Var}(N_{b,s,j}) = \sum_{s,j}\text{Var}(E_{b,s,j}) - \mathbb{E}\|\mathbf{E}^{\text{str}}_b\|_F^2$.
\end{enumerate}
\end{proposition}
The covariance statement is the first statement of Theorem~\ref{thm:centering} applied to every row and every column of the error, and under per-token thresholds the rows carry nearly all of the removed covariance, by Corollary~\ref{cor:structured_topologies}.

\textbf{Relation to dithering.} The correction is a control variate, akin to bias correction in post-training quantization \cite{nagel2019data}. SR rounds to nearest after adding the dither $\Delta(r - \frac{1}{2})$; because the dither is shared, the error it causes is known and can be subtracted, as in subtractive dither \cite{gray1993dithered}. Correlated quantization \cite{suresh2022correlated} likewise gains from shared randomness. The diagnostics agree in direction: with the two-way correction, $\Phi_{\text{2D}}$ has a bottom-layer activation-gradient error of 0.476 vs.\ 0.572 for independent SR (Table~\ref{tab:grad_diag}).

\subsubsection{Cost and the dX-Only Variant}
The correction reuses $\mathbf{E}$ and adds only per-example reductions and small products. With the batch sharded, its weight-gradient term all-gathers small per-example factors such as $\bar{\mathbf{x}}_b$ and $\bar{\mathbf{e}}_b$ instead of reducing a weight-sized \texttt{f32} partial, which lowers the \texttt{f32} all-reduce volume of a compiled two-layer model from 422.6\,MiB to 3.1\,MiB per step, at the cost of about 8\,MiB of small factor all-gathers. Under $\Phi_{\text{2D}}$, the errors of different tokens are uncorrelated, so centering over tokens finds almost no covariance to remove (Corollary~\ref{cor:structured_topologies}), and the activation-gradient term (\ref{eq:s3r_opt_dx}) alone keeps nearly all of the benefit, including the removal of the shared term of every token's activation-gradient error. Intuitively, the token offsets are independent under $\Phi_{\text{2D}}$, so they largely average out in the weight gradient, a sum over tokens, but not in each token's activation gradient (Fig.~\ref{fig:ec_example}). This dX-only correction matches the two-way correction within seed noise (Section~\ref{subsec:convergence}) and adds 1.1--1.8\% to the S3R step time of the three dense models. $\Phi_{\text{bcast}}$, whose tokens share thresholds, needs both terms; Section~\ref{subsec:results} reports the remaining costs.

\section{Experiments}
\label{sec:experiments}

We report model quality first (Section~\ref{subsec:convergence}) and speed second (Section~\ref{subsec:results}), because a faster configuration is useful only if it trains as well as its baseline.

\subsection{Setup}
\label{subsec:setup}
All experiments run on Google Cloud TPU v5p with JAX/XLA \cite{bradbury2018jax}. We use three workloads.
\begin{enumerate}
    \item \textbf{Convergence runs.} A 12-layer Llama-style decoder is trained on English C4 \cite{raffel2020exploring} for 2000 steps ($S = 2048$, $B = 8$, 32.8M tokens). We report the mean and sample standard deviation of the final validation loss over three seeds, with each configuration at its best learning rate on seed 0.
    \item \textbf{Single-block FFN benchmarks.} Forward and backward passes of one SwiGLU FFN block at nine shapes ($W_1$--$W_9$) taken from models of 1B to 70B parameters (Table~\ref{tab:ffn_step_time}).
    \item \textbf{Full training steps.} MaxText training steps of five models from four families \cite{grattafiori2024llama, qwen2025qwen3, gemma2025gemma3, liu2024deepseek}: the dense Llama 3.1-8B, Qwen3-8B, and Gemma 3-12B, and the mixture-of-experts (MoE) models Qwen3-30B-A3B and DeepSeek-V2-Lite (16B). The devices form a two-dimensional mesh with Fully Sharded Data Parallelism (FSDP) \cite{rajbhandari2020zero} and Tensor Parallelism (TP) \cite{shoeybi2019megatron}, orchestrated via Pathways \cite{barham2022pathways}. Each chip processes 4 synthetic sequences of 4096 tokens per step, with full rematerialization and flash attention \cite{dao2022flashattention}. Step times are medians over 20 steps after 12 warm-up steps and differ by at most 0.6\% between repeated runs.
\end{enumerate}

\subsection{Compared Configurations}
All INT8 configurations quantize as in Section~\ref{sec:method}: RTNE in the forward pass, one per-tensor step $\Delta$ for the backward gradient, and int32 accumulation. In the full training steps, each GEMM rounds its output to \texttt{bf16} per shard, so the TP partial sums are reduced in \texttt{bf16}, as in the BF16 baseline. The flash-attention products, the embedding, and the logits projection stay in \texttt{bf16}. They differ only in how the upstream gradient is rounded:
\begin{itemize}
    \item \textbf{BF16} and \textbf{INT8 RTNE}: the 16-bit baseline and deterministic rounding.
    \item \textbf{Independent SR (Threefry)}: one 32-bit random threshold per element from JAX's default Threefry generator, which S3R also uses. In full training steps, it quantizes the upstream gradient once per backward product, as in AQT.
    \item \textbf{S3R ($\Phi_{\text{2D}}$)}, the default with one threshold per token, and \textbf{S3R ($\Phi_{\text{bcast}}$)}, the ablation with one threshold per example.
    \item \textbf{S3R-EC, two-way} and \textbf{S3R-EC, dX-only}: S3R with the correction of Section~\ref{subsec:s3r_opt}, applied to both backward products or to the activation gradient only.
\end{itemize}

\subsection{Model Quality}
\label{subsec:convergence}

\begin{figure}[t]
\centering
\includegraphics[width=\columnwidth]{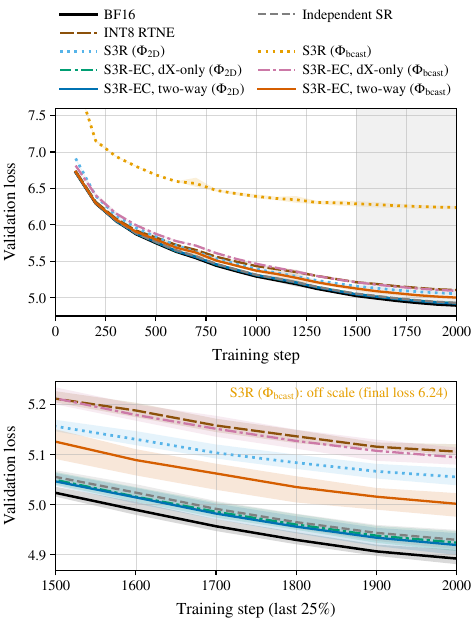}
\caption{Validation loss versus training step for all configurations of Table~\ref{tab:convergence}, each at its best learning rate. Lines show the mean over 3 seeds, and shaded bands span the seed minimum to maximum. The bottom panel enlarges the last 25\% of steps. S3R ($\Phi_{\text{bcast}}$) lies off this scale (final loss 6.24).}
\label{fig:convergence_curves}
\vspace{-6pt}
\end{figure}

\begin{table}[t!]
\centering
\caption{Final Validation Loss of the 12-Layer Model after 2000 Steps (32.8M C4 Tokens) with 8-Bit Per-Tensor Backward Gradients. Mean $\pm$ Sample Std over 3 Seeds at Each Configuration's Best Learning Rate, Chosen on Seed 0. All Configurations Use AdamW. $^\dagger$Lowest Learning Rate Tried}
\label{tab:convergence}
\resizebox{\columnwidth}{!}{%
\begin{tabular}{@{} l c c r @{}}
\toprule
Configuration & Best LR & Val.\ loss & Gap to BF16 \\
\midrule
BF16 & $6 \times 10^{-4}$ & $4.892 \pm 0.015$ & --- \\
INT8 RTNE & $6 \times 10^{-4}$ & $5.106 \pm 0.013$ & $+0.214$ \\
Independent SR & $6 \times 10^{-4}$ & $4.929 \pm 0.019$ & $+0.037$ \\
\midrule
S3R ($\Phi_{\text{2D}}$) & $3 \times 10^{-4}$ & $5.055 \pm 0.017$ & $+0.163$ \\
S3R-EC, dX-only ($\Phi_{\text{2D}}$) & $6 \times 10^{-4}$ & $4.924 \pm 0.021$ & $+0.032$ \\
S3R-EC, two-way ($\Phi_{\text{2D}}$) & $6 \times 10^{-4}$ & $4.919 \pm 0.022$ & $+0.027$ \\
\midrule
S3R ($\Phi_{\text{bcast}}$) & $3.75 \times 10^{-5}$\,$^\dagger$ & $6.239 \pm 0.026$ & $+1.347$ \\
S3R-EC, dX-only ($\Phi_{\text{bcast}}$) & $6 \times 10^{-4}$ & $5.094 \pm 0.018$ & $+0.202$ \\
S3R-EC, two-way ($\Phi_{\text{bcast}}$) & $6 \times 10^{-4}$ & $5.002 \pm 0.024$ & $+0.109$ \\
\bottomrule
\end{tabular}%
}
\vspace{-6pt}
\end{table}

\begin{table}[t!]
\centering
\caption{Relative Stochastic Gradient Error (Frobenius Norm Relative to the BF16 Gradient) on a BF16 Checkpoint at Step 1000 (LR $6 \times 10^{-4}$, Seed 0), Pooled over 4--9 Fixed Training Batches. Layer 0 Is the Bottom of the 12-Layer Model. The dW-Only Correction (Corr.) Keeps Only the Token-Mean Term of the Weight-Gradient Correction. The Last Two Rows Choose the Thresholds of the Two Backward Products Separately and Use the dW-Only Correction}
\label{tab:grad_diag}
\resizebox{\columnwidth}{!}{%
\begin{tabular}{@{} l c c c c @{}}
\toprule
 & \multicolumn{2}{c}{dX error} & dW error & Emb. \\
Configuration & layer 0 & layer 11 & layer 0 & error \\
\midrule
Independent SR & 0.572 & 0.111 & 0.456 & 0.457 \\
\midrule
S3R ($\Phi_{\text{2D}}$) & 0.877 & 0.228 & 0.691 & 0.703 \\
S3R ($\Phi_{\text{2D}}$), dW-only corr. & 0.882 & 0.228 & 0.694 & 0.706 \\
S3R-EC, two-way ($\Phi_{\text{2D}}$) & 0.476 & 0.103 & 0.377 & 0.378 \\
\midrule
S3R ($\Phi_{\text{bcast}}$) & 3.157 & 0.313 & 5.141 & 5.285 \\
S3R ($\Phi_{\text{bcast}}$), dW-only corr. & 3.150 & 0.313 & 4.529 & 5.262 \\
S3R-EC, two-way ($\Phi_{\text{bcast}}$) & 0.521 & 0.105 & 0.551 & 0.558 \\
\midrule
dX $\Phi_{\text{2D}}$, dW $\Phi_{\text{bcast}}$ & 0.893 & 0.227 & 0.746 & 0.714 \\
dX $\Phi_{\text{bcast}}$, dW $\Phi_{\text{2D}}$ & 3.494 & 0.261 & 5.979 & 6.849 \\
\bottomrule
\end{tabular}%
}
\vspace{-6pt}
\end{table}

Table~\ref{tab:convergence} lists the final validation loss of every configuration, and Fig.~\ref{fig:convergence_curves} shows the loss curves of all of them.

\textbf{SR is still needed at 8 bits.} With deterministic rounding, INT8 RTNE ends 0.176 above independent SR (Welch \mbox{$t \approx 13.2$}) and 0.214 above BF16 (Table~\ref{tab:convergence}).

\textbf{Sharing hurts, mainly through the activation gradient.} Without a correction, S3R ($\Phi_{\text{2D}}$) ends 0.126 above independent SR. S3R ($\Phi_{\text{bcast}}$) does not train usefully, and its loss keeps improving as the learning rate decreases, down to the lowest rate we tried (Table~\ref{tab:convergence}). A causal ablation in the last two rows of Table~\ref{tab:grad_diag} locates the damage. These rows quantize $\mathbf{G}$ separately for the two backward products. Per-example thresholds on the activation-gradient product alone reproduce the gradient error of S3R ($\Phi_{\text{bcast}}$), and per-token thresholds there remove almost all of it, whichever thresholds the weight-gradient product uses.

\textbf{With per-token thresholds, S3R-EC restores quality within seed noise.} S3R-EC, two-way ($\Phi_{\text{2D}}$) reaches \mbox{$4.919 \pm 0.022$}, against \mbox{$4.929 \pm 0.019$} for independent SR. Its seed band overlaps that of independent SR over the last 25\% of training, and its best learning rate returns to that of BF16. With per-example thresholds, the correction closes most of the gap but not all of it. It stays 0.073 above independent SR (Welch $t \approx 4.1$) and 0.109 above BF16 (Table~\ref{tab:convergence}, Fig.~\ref{fig:convergence_curves}).

\textbf{The dX-only correction suffices for $\Phi_{\text{2D}}$ but not for $\Phi_{\text{bcast}}$.} With per-token thresholds, the dX-only and two-way variants end within seed noise of each other. Because it skips the weight-gradient term, the dX-only variant is also cheaper in full training steps (Section~\ref{subsec:results}). With per-example thresholds, correcting the activation gradient alone recovers most of the loss, but the dX-only variant stays 0.092 above the two-way correction (Welch $t = 5.4$; Table~\ref{tab:convergence}). This pattern agrees with the analysis. Per-token thresholds leave the errors of different tokens uncorrelated, so the weight-gradient half of the correction, which centers over tokens, has almost no covariance to remove; per-example thresholds correlate the tokens, and this half then removes their covariance too (Corollary~\ref{cor:structured_topologies}).

\textbf{The correction lowers the gradient error itself, whereas the rounding variance barely affects the optimizer.} Table~\ref{tab:grad_diag} measures every configuration on the same BF16 checkpoint. Under $\Phi_{\text{2D}}$, the two-way correction brings every listed error below that of independent SR, for example to 0.48 against 0.57 in the bottom-layer activation gradient. Under $\Phi_{\text{bcast}}$, it lowers the bottom-layer activation-gradient error from 3.16 to 0.52, below that of independent SR, although the weight-gradient error stays higher. The rounding variance makes up at most 0.094 of the AdamW second-moment estimate $v_t$ late in training, so it barely changes the per-coordinate step size of AdamW.

\textbf{Limits.} Three seeds per configuration resolve differences of about 0.05 or more. Smaller differences remain undetermined, including the one between independent SR and the two-way S3R-EC with $\Phi_{\text{2D}}$. The learning rate of each configuration was selected on seed 0, which is also part of the mean.

\subsection{Throughput and Memory}
\label{subsec:results}
\begin{table}[t!]
\centering
\caption{Speedup over BF16 of One SwiGLU FFN Block (Forward + Backward, per Device) on TPU v5p, and Extra Device Time of S3R-EC over S3R ($\Phi_{\text{2D}}$). $B = 8$ ($4$ for $W_9$), $S = 2048$; Median over Repeated Steps. S3R-EC: Mean of Two Further Runs, Each against S3R in the Same Run.}
\label{tab:ffn_step_time}
\footnotesize
\setlength{\tabcolsep}{0pt}
\begin{tabular*}{\columnwidth}{@{\extracolsep{\fill}} l c c c c r r @{}}
\toprule
 & & \multicolumn{3}{c}{Speedup over BF16 ($\times$)} & \multicolumn{2}{c}{\begin{tabular}[b]{@{}c@{}}S3R-EC, extra time\\over S3R (\%)\end{tabular}} \\
\cmidrule(lr){3-5} \cmidrule(l){6-7}
Shape & $(H, D_{\text{ffn}})$ & \begin{tabular}[b]{@{}c@{}}INT8\\RTNE\end{tabular} & \begin{tabular}[b]{@{}c@{}}Indep.\ SR\\(Threefry)\end{tabular} & \begin{tabular}[b]{@{}c@{}}S3R\\($\Phi_{\text{2D}}$)\end{tabular} & Two-way & dX-only \\
\midrule
$W_1$ & $(512, 2048)$ & 1.12 & 0.58 & 1.08 & $+13.0$ & $+14.8$ \\
$W_2$ & $(1024, 4096)$ & 1.45 & 0.91 & 1.44 & $+7.3$ & $+9.6$ \\
$W_3$ & $(2048, 8192)$ & 1.73 & 1.29 & 1.73 & $-0.6$ & $+1.7$ \\
$W_4$ & $(3072, 12288)$ & 1.92 & 1.54 & 1.92 & $+4.8$ & $+4.1$ \\
$W_5$ & $(3840, 15360)$ & 1.84 & 1.53 & 1.83 & $+2.8$ & $+3.3$ \\
$W_6$ & $(4096, 12288)$ & 1.89 & 1.60 & 1.90 & $+2.7$ & $+3.1$ \\
$W_7$ & $(4096, 14336)$ & 1.87 & 1.56 & 1.87 & $+3.4$ & $+3.1$ \\
$W_8$ & $(5120, 17920)$ & 1.82 & 1.58 & 1.81 & $+0.7$ & $+2.7$ \\
$W_9$ & $(8192, 28672)$ & 1.81 & 1.66 & 1.81 & $+3.8$ & $+2.2$ \\
\bottomrule
\end{tabular*}
\vspace{-6pt}
\end{table}
\begin{table*}[t!]
\centering
\caption{Full Training Step (Forward, Backward, and AdamW Update) on TPU v5p with \texttt{bf16}-Output INT8 GEMMs. Median Step Time (ms) over 20 Steps after 12 Warm-Up Steps; Per-Device Batch 4, $S = 4096$. In Parentheses: Speedup over BF16 on the Dense Models. S3R-EC, Two-Way, Uses the Factor All-Gather; on the Dense GEMMs, It Also Computes a Per-Channel Rounding-Variance Statistic That the Update Does Not Use, So Its Times Are Slightly Pessimistic. On the MoE Models, the Expert GEMMs Receive Only the Channel-Mean Part of the Correction.}
\label{tab:e2e_step}
\small
\setlength{\tabcolsep}{0pt}
\begin{tabular*}{\textwidth}{@{\extracolsep{\fill}} l r r r r r @{}}
\toprule
 & \multicolumn{3}{c}{Dense models} & \multicolumn{2}{c}{MoE models} \\
\cmidrule(lr){2-4} \cmidrule(l){5-6}
Configuration & Llama 3.1-8B & Qwen3-8B & Gemma 3-12B & Qwen3-30B-A3B & DeepSeek-V2-Lite (16B) \\
\midrule
BF16 & 3862 & 4018 & 6550 & 6454 & 3659 \\
INT8 RTNE & 2926 ($1.320\times$) & 3110 ($1.292\times$) & 4822 ($1.358\times$) & 7305 & 3842 \\
Independent SR (Threefry) & 3350 ($1.153\times$) & 3607 ($1.114\times$) & 5526 ($1.185\times$) & 7914 & 4269 \\
S3R ($\Phi_{\text{2D}}$) & 2903 ($1.331\times$) & 3104 ($1.294\times$) & 4853 ($1.350\times$) & 7286 & 3840 \\
S3R-EC, two-way & 3008 ($1.284\times$) & 3243 ($1.239\times$) & 4961 ($1.320\times$) & 7621 & 4029 \\
S3R-EC, dX-only & 2935 ($1.316\times$) & 3160 ($1.272\times$) & 4911 ($1.334\times$) & 7548 & 3984 \\
\midrule
\multicolumn{6}{@{}l}{\textit{Speedup over independent SR (Threefry)}} \\
S3R & $1.154\times$ & $1.162\times$ & $1.139\times$ & $1.086\times$ & $1.112\times$ \\
S3R-EC, two-way & $1.114\times$ & $1.112\times$ & $1.114\times$ & $1.038\times$ & $1.060\times$ \\
S3R-EC, dX-only & $1.141\times$ & $1.141\times$ & $1.125\times$ & $1.049\times$ & $1.072\times$ \\
\bottomrule
\end{tabular*}
\vspace{-6pt}
\end{table*}

\textbf{INT8 nearly halves the GEMM time, and vector work decides how much of this saving remains.} The S8 GEMMs run on the MXU and take the same time in every uncorrected INT8 configuration, $47.55$--$47.63$\,ms for one FFN block at $W_9$ against $90.29$\,ms for BF16. Scaling, rounding, threshold generation and the correction run on the VPU and do not get cheaper with a narrower format. We convert extra device time to int32 VPU operations per upstream-gradient element at the measured $14.3 \times 10^{12}$ operations/s.

\textbf{Per-element Threefry is the expensive step, and S3R removes it.} Independent SR hashes one 32-bit word per gradient element, which costs 120--133 operations per element over INT8 RTNE at every shape and takes 56.5\% of the busy time at $W_1$ and 9.9\% at $W_9$. It is therefore slower than BF16 at $W_1$ and $W_2$ (Table~\ref{tab:ffn_step_time}). S3R ($\Phi_{\text{2D}}$) hashes one word per four tokens and broadcasts the thresholds in vector registers inside the quantizer fusion, so they are never written to memory. It costs $-2$ to $+6$ operations per element, runs within 1\% of INT8 RTNE at $W_2$--$W_9$ and 4.3\% slower at $W_1$, and is $1.08$--$1.92\times$ faster than BF16.

\textbf{S3R-EC adds small reductions and products.} It reuses the rounding error inside the quantizer fusion, and its per-token and per-channel reductions, matrix-vector products and rank-1 updates leave the GEMM time nearly unchanged. On a single block, both variants add 15--23 operations per element at most shapes and 8--14\% temporary memory. This is 13--15\% of the device time of S3R at $W_1$, 7--10\% at $W_2$ and at most 5\% from $W_3$ on, so only $W_1$ becomes slower than BF16 ($0.94$--$0.96\times$; Table~\ref{tab:ffn_step_time}). The two-way variant does more arithmetic but is not always slower, because both variants change how XLA fuses the GEMMs.

\textbf{On the dense models, the GEMM saving reaches the full step.} In a Llama 3.1-8B step, INT8 RTNE spends 1182\,ms in the GEMMs against 2236\,ms for BF16, and 249 against 174\,ms in vector fusions. S3R matches INT8 RTNE to within 1\% on all five models and is $1.29$--$1.35\times$ faster than BF16 on the dense models (Table~\ref{tab:e2e_step}), but without a correction it trails independent SR in quality (Section~\ref{subsec:convergence}). Independent SR also quantizes the upstream gradient twice in full steps, doubling the quantization work and the S8 writes, and is only $1.11$--$1.19\times$ faster than BF16. S3R-EC, two-way, is $1.24$--$1.32\times$ faster than BF16 and $1.11\times$ faster than independent SR. Its weight-gradient term needs extra reductions and small factor all-gathers (Section~\ref{subsec:s3r_opt}) and adds 2.2--4.5\% to the S3R step; the dX-only correction skips it and adds 1.1--1.8\%. By XLA's compile-time estimate, S3R needs the same full-step device memory as RTNE on the dense models, and S3R-EC, two-way, at most 0.4\,GB more.

\textbf{On the MoE models, S3R and S3R-EC are also faster than independent SR.} S3R is $1.09$--$1.11\times$ faster, runs as fast as INT8 RTNE, and needs 0.4 and 0.7\,GB more device memory than RTNE. With the expert-path correction fused into Pallas kernels over the expert-sorted rows, the two-way correction adds 4.6\% and 4.9\% to the S3R step, so S3R-EC is $1.04\times$ and $1.06\times$ faster than independent SR; with the dX-only correction, it is $1.05\times$ and $1.07\times$ faster.
\section{Conclusions}
\label{sec:conclusion}

Stochastic rounding keeps low-precision gradients unbiased, but standard SR draws one random threshold per element on the vector units, which can make INT8 training slower than BF16 training. Sharing thresholds removes this cost but shifts the rounding errors of all elements that share a threshold by a common offset. We showed that this offset harms training mainly through the activation gradient, which carries it to every lower layer, so it must be removed where it is produced. S3R quantizes each upstream gradient once for both backward products and uses one threshold per token. S3R-EC subtracts the structured part of the known rounding error from the activation and weight gradients, keeps them unbiased, and leaves the optimizer unchanged. In pre-training, uncorrected sharing degrades quality, whereas S3R-EC removes this loss and reaches the quality of independent SR. At this quality, S3R-EC makes full training steps of three dense LLMs $1.24$--$1.32\times$ faster than BF16 and about $1.11\times$ faster than independent SR.

\bibliographystyle{IEEEtran}
\bibliography{example_paper}

@IEEEtranBSTCTL{IEEEexample:BSTcontrol,
  CTLuse_forced_etal = "yes",
  CTLmax_names_forced_etal = "6",
  CTLnames_show_etal = "1"
}

@article{han2015learning,
  title={Learning both weights and connections for efficient neural network},
  author={Han, Song and Pool, Jeff and Tran, John and Dally, William},
  journal={Advances in neural information processing systems},
  volume={28}, 
  year={2015}
}

@article{raffel2020exploring,
  title={Exploring the limits of transfer learning with a unified text-to-text transformer},
  author={Raffel, Colin and Shazeer, Noam and Roberts, Adam and Lee, Katherine and Narang, Sharan and Matena, Michael and Zhou, Yanqi and Li, Wei and Liu, Peter J},
  journal={Journal of machine learning research},
  volume={21},
  number={140},
  pages={1--67},
  year={2020}
}

@article{liu2024deepseek,
  title={Deepseek-v3 technical report},
  author={Liu, Aixin and Feng, Bei and Xue, Bing and Wang, Bingxuan and Wu, Bochao and Lu, Chengda and Zhao, Chenggang and Deng, Chengqi and Zhang, Chenyu and Ruan, Chong and others},
  journal={arXiv preprint arXiv:2412.19437},
  year={2024}
}

@article{grattafiori2024llama,
  title={The {Llama} 3 herd of models},
  author={Grattafiori, Aaron and Dubey, Abhimanyu and Jauhri, Abhinav and Pandey, Abhinav and Kadian, Abhishek and Al-Dahle, Ahmad and Letman, Aiesha and Mathur, Akhil and Schelten, Alan and Vaughan, Alex and others},
  journal={arXiv preprint arXiv:2407.21783},
  year={2024}
}

@inproceedings{gupta2015deep,
  title={Deep learning with limited numerical precision},
  author={Gupta, Suyog and Agrawal, Ankur and Gopalakrishnan, Kailash and Narayanan, Pritish},
  booktitle={International conference on machine learning},
  pages={1737--1746},
  year={2015},
  organization={PMLR}
}

@article{wang2018training,
  title={Training deep neural networks with 8-bit floating point numbers},
  author={Wang, Naigang and Choi, Jungwook and Brand, Daniel and Chen, Chia-Yu and Gopalakrishnan, Kailash},
  journal={Advances in neural information processing systems},
  volume={31},
  year={2018}
}

@article{micikevicius2022fp8,
  title={Fp8 formats for deep learning},
  author={Micikevicius, Paulius and Stosic, Dusan and Burgess, Neil and Cornea, Marius and Dubey, Pradeep and Grisenthwaite, Richard and Ha, Sangwon and Heinecke, Alexander and Judd, Patrick and Kamalu, John and others},
  journal={arXiv preprint arXiv:2209.05433},
  year={2022}
}

@inproceedings{wiedemann2020dithered,
  title={Dithered backprop: A sparse and quantized backpropagation algorithm for more efficient deep neural network training},
  author={Wiedemann, Simon and Mehari, Temesgen and Kepp, Kevin and Samek, Wojciech},
  booktitle={2020 IEEE/CVF Conference on Computer Vision and Pattern Recognition Workshops (CVPRW)},
  pages={3096--3104},
  year={2020},
  organization={IEEE}
}

@inproceedings{yuan2022you,
  title={You already have it: A generator-free low-precision dnn training framework using stochastic rounding},
  author={Yuan, Geng and Chang, Sung-En and Jin, Qing and Lu, Alec and Li, Yanyu and Wu, Yushu and Kong, Zhenglun and Xie, Yanyue and Dong, Peiyan and Qin, Minghai and others},
  booktitle={European Conference on Computer Vision},
  pages={34--51},
  year={2022},
  organization={Springer}
}

@article{brown2020language,
  title={Language models are few-shot learners},
  author={Brown, Tom and Mann, Benjamin and Ryder, Nick and Subbiah, Melanie and Kaplan, Jared D and Dhariwal, Prafulla and Neelakantan, Arvind and Shyam, Pranav and Sastry, Girish and Askell, Amanda and others},
  journal={Advances in neural information processing systems},
  volume={33},
  pages={1877--1901},
  year={2020}
}

@article{touvron2023llama,
  title={Llama 2: Open foundation and fine-tuned chat models},
  author={Touvron, Hugo and Martin, Louis and Stone, Kevin and Albert, Peter and Almahairi, Amjad and Babaei, Yasmine and Bashlykov, Nikolay and Batra, Soumya and Bhargava, Prajjwal and Bhosale, Shruti and others},
  journal={arXiv preprint arXiv:2307.09288},
  year={2023}
}

@article{sun2020ultra,
  title={Ultra-low precision {4-bit} training of deep neural networks},
  author={Sun, Xiao and Wang, Naigang and Chen, Chia-Yu and Ni, Jiamin and Agrawal, Ankur and Cui, Xiaodong and Venkataramani, Swagath and El Maghraoui, Kaoutar and Srinivasan, Vijayalakshmi Viji and Gopalakrishnan, Kailash},
  journal={Advances in Neural Information Processing Systems},
  volume={33},
  pages={1796--1807},
  year={2020}
}

@article{hubara2017quantized,
  title={Quantized neural networks: Training neural networks with low precision weights and activations},
  author={Hubara, Itay and Courbariaux, Matthieu and Soudry, Daniel and El-Yaniv, Ran and Bengio, Yoshua},
  journal={journal of machine learning research},
  volume={18},
  number={187},
  pages={1--30},
  year={2018}
}

@article{qwen2025qwen3,
  title={Qwen3 technical report},
  author={Yang, An and Li, Anfeng and Yang, Baosong and Zhang, Beichen and Hui, Binyuan and Zheng, Bo and Yu, Bowen and Gao, Chang and Huang, Chengen and Lv, Chenxu and others},
  journal={arXiv preprint arXiv:2505.09388},
  year={2025}
}

@article{gemma2025gemma3,
  title={Gemma 3 technical report},
  author={Team, Gemma and Kamath, Aishwarya and Ferret, Johan and Pathak, Shreya and Vieillard, Nino and Merhej, Ramona and Perrin, Sarah and Matejovicova, Tatiana and Ram{\'e}, Alexandre and Rivi{\`e}re, Morgane and others},
  journal={arXiv preprint arXiv:2503.19786},
  year={2025}
}

@article{loshchilov2019decoupled,
  title={Decoupled weight decay regularization},
  author={Loshchilov, Ilya and Hutter, Frank},
  journal={arXiv preprint arXiv:1711.05101},
  year={2017}
}

@inproceedings{han2016eie,
  title={{EIE}: Efficient Inference Engine on Compressed Deep Neural Network},
  author={Han, Song and Liu, Xingyu and Mao, Huizi and Pu, Jing and Pedram, Ardavan and Horowitz, Mark A and Dally, William J},
  booktitle={ACM/IEEE 43rd Annual International Symposium on Computer Architecture (ISCA)},
  pages={243--254},
  year={2016},
  organization={IEEE}
}

@inproceedings{xiao2023smoothquant,
  title={Smoothquant: Accurate and efficient post-training quantization for large language models},
  author={Xiao, Guangxuan and Lin, Ji and Seznec, Mickael and Wu, Hao and Demouth, Julien and Han, Song},
  booktitle={International conference on machine learning},
  pages={38087--38099},
  year={2023},
  organization={PMLR}
}

@inproceedings{lin2024awq,
  title={{AWQ}: Activation-Aware Weight Quantization for On-Device {LLM} Compression and Acceleration},
  author={Lin, Ji and Tang, Jiaming and Tang, Haotian and Yang, Shang and Chen, Wei-Ming and Wang, Wei-Chen and Xiao, Guangxuan and Dang, Xingyu and Gan, Chuang and Han, Song},
  booktitle={Proceedings of Machine Learning and Systems (MLSys)},
  volume={6},
  pages={87--100},
  year={2024}
}

@article{lin2024qserve,
  title={Qserve: W4a8kv4 quantization and system co-design for efficient llm serving},
  author={Lin, Yujun and Tang, Haotian and Yang, Shang and Zhang, Zhekai and Xiao, Guangxuan and Gan, Chuang and Han, Song},
  journal={Proceedings of Machine Learning and Systems},
  volume={7},
  year={2025}
}

@article{xi2024jetfire,
  title={Jetfire: Efficient and accurate transformer pretraining with int8 data flow and per-block quantization},
  author={Xi, Haocheng and Chen, Yuxiang and Zhao, Kang and Teh, Kai Jun and Chen, Jianfei and Zhu, Jun},
  journal={arXiv preprint arXiv:2403.12422},
  year={2024}
}

@inproceedings{xi2024coat,
  title={COAT: Compressing optimizer states and activations for memory-efficient FP8 training},
  author={Xi, Haocheng and Cai, Han and Zhu, Ligeng and Lu, Yao and Keutzer, Kurt and Chen, Jianfei and Han, Song},
  booktitle={International Conference on Learning Representations},
  volume={2025},
  pages={42989--43011},
  year={2025}
}

@article{barham2022pathways,
  title={Pathways: Asynchronous distributed dataflow for ml},
  author={Barham, Paul and Chowdhery, Aakanksha and Dean, Jeff and Ghemawat, Sanjay and Hand, Steven and Hurt, Daniel and Isard, Michael and Lim, Hyeontaek and Pang, Ruoming and Roy, Sudip and others},
  journal={Proceedings of Machine Learning and Systems},
  volume={4},
  pages={430--449},
  year={2022}
}

@inproceedings{ethayarajh2019contextual,
  title={How Contextual are Contextualized Word Representations? {C}omparing the Geometry of {BERT}, {ELMo}, and {GPT-2} Embeddings},
  author={Ethayarajh, Kawin},
  booktitle={Proceedings of the Conference on Empirical Methods in Natural Language Processing and the International Joint Conference on Natural Language Processing (EMNLP-IJCNLP)},
  pages={55--65},
  year={2019}
}

@article{connolly2021stochastic,
  title={Stochastic Rounding and Its Probabilistic Backward Error Analysis},
  author={Connolly, Michael P. and Higham, Nicholas J. and Mary, Theo},
  journal={SIAM Journal on Scientific Computing},
  volume={43},
  number={1},
  pages={A566--A585},
  year={2021},
  doi={10.1137/20M1334796}
}

@article{croci2022stochastic,
  title={Stochastic rounding: implementation, error analysis and applications},
  author={Croci, Matteo and Fasi, Massimiliano and Higham, Nicholas J and Mary, Theo and Mikaitis, Mantas},
  journal={Royal Society Open Science},
  volume={9},
  number={3},
  pages={211631},
  year={2022}
}

@article{elarar2025probabilistic,
  title={Probabilistic Error Analysis of Limited-Precision Stochastic Rounding},
  author={El Arar, El-Mehdi and Fasi, Massimiliano and Filip, Silviu-Ioan and Mikaitis, Mantas},
  journal={SIAM Journal on Scientific Computing},
  volume={47},
  number={5},
  pages={B1227--B1249},
  year={2025},
  doi={10.1137/24M1681458}
}

@inproceedings{fitzgibbon2025stochastic,
  title={On Stochastic Rounding with Few Random Bits},
  author={Fitzgibbon, Andrew and Felix, Stephen},
  booktitle={IEEE 32nd Symposium on Computer Arithmetic (ARITH)},
  pages={133--140},
  year={2025},
  doi={10.1109/ARITH64983.2025.00029}
}

@inproceedings{mikaitis2021stochastic,
  title={Stochastic rounding: Algorithms and hardware accelerator},
  author={Mikaitis, Mantas},
  booktitle={2021 International Joint Conference on Neural Networks (IJCNN)},
  pages={1--6},
  year={2021},
  organization={IEEE}
}

@article{dexter2024stochastic,
  title={Stochastic rounding implicitly regularizes tall-and-thin matrices},
  author={Dexter, Gregory and Boutsikas, Christos and Ma, Linkai and Ipsen, Ilse CF and Drineas, Petros},
  journal={SIAM Journal on Matrix Analysis and Applications},
  volume={46},
  number={1},
  pages={341--369},
  year={2025},
  publisher={SIAM}
}

@article{gray1993dithered,
  title={Dithered Quantizers},
  author={Gray, Robert M. and Stockham, Thomas G.},
  journal={IEEE Transactions on Information Theory},
  volume={39},
  number={3},
  pages={805--812},
  year={1993},
  doi={10.1109/18.256489}
}

@inproceedings{suresh2022correlated,
  title={Correlated quantization for distributed mean estimation and optimization},
  author={Suresh, Ananda Theertha and Sun, Ziteng and Ro, Jae and Yu, Felix},
  booktitle={International Conference on Machine Learning},
  pages={20856--20876},
  year={2022},
  organization={PMLR}
}

@inproceedings{nagel2019data,
  title={Data-free quantization through weight equalization and bias correction},
  author={Nagel, Markus and Baalen, Mart van and Blankevoort, Tijmen and Welling, Max},
  booktitle={Proceedings of the IEEE/CVF international conference on computer vision},
  pages={1325--1334},
  year={2019}
}

@inproceedings{chen2020statistical,
  title={A Statistical Framework for Low-bitwidth Training of Deep Neural Networks},
  author={Chen, Jianfei and Gai, Yu and Yao, Zhewei and Mahoney, Michael W. and Gonzalez, Joseph E.},
  booktitle={Advances in Neural Information Processing Systems (NeurIPS)},
  volume={33},
  year={2020}
}

@inproceedings{ozkara2025stochastic,
  title={Stochastic Rounding for {LLM} Training: Theory and Practice},
  author={Ozkara, Kaan and Yu, Tao and Park, Youngsuk},
  booktitle={International Conference on Artificial Intelligence and Statistics (AISTATS)},
  volume={258},
  year={2025},
  organization={PMLR}
}

@article{peng2023fp8lm,
  title={{FP8-LM}: Training {FP8} Large Language Models},
  author={Peng, Houwen and Wu, Kan and Wei, Yixuan and Zhao, Guoshuai and Yang, Yuxiang and Liu, Ze and Xiong, Yifan and Yang, Ziyue and Ni, Bolin and Hu, Jingcheng and others},
  journal={arXiv preprint arXiv:2310.18313},
  year={2023}
}

@article{liu2024deepseekv3,
  title={{DeepSeek-V3} Technical Report},
  author={{DeepSeek-AI} and Liu, Aixin and Feng, Bei and Xue, Bing and Wang, Bingxuan and Wu, Bochao and others},
  journal={arXiv preprint arXiv:2412.19437},
  year={2024}
}

@inproceedings{fishman2025scaling,
  title={Scaling {FP8} Training to Trillion-Token {LLMs}},
  author={Fishman, Maxim and Chmiel, Brian and Banner, Ron and Soudry, Daniel},
  booktitle={International Conference on Learning Representations (ICLR)},
  year={2025}
}

@article{nvidia2025nvfp4,
  title={Pretraining large language models with nvfp4},
  author={Abecassis, Felix and Agrusa, Anjulie and Ahn, Dong and Alben, Jonah and Alborghetti, Stefania and Andersch, Michael and Arayandi, Sivakumar and Bjorlin, Alexis and Blakeman, Aaron and Briones, Evan and others},
  journal={arXiv preprint arXiv:2509.25149},
  year={2025}
}

@inproceedings{tseng2025training,
  title={Training {LLMs} with {MXFP4}},
  author={Tseng, Albert and Yu, Tao and Park, Youngsuk},
  booktitle={International Conference on Artificial Intelligence and Statistics (AISTATS)},
  volume={258},
  year={2025},
  organization={PMLR}
}

@article{castro2025quartet,
  title={Quartet: Native fp4 training can be optimal for large language models},
  author={Castro, Roberto and Panferov, Andrei and Tabesh, Rush and Sieberling, Oliver and Chen, Jiale and Nikdan, Mahdi and Ashkboos, Saleh and Alistarh, Dan},
  journal={Advances in Neural Information Processing Systems},
  volume={38},
  pages={43552--43572},
  year={2026}
}

@inproceedings{wortsman2023stable,
  title={Stable and Low-Precision Training for Large-Scale Vision-Language Models},
  author={Wortsman, Mitchell and Dettmers, Tim and Zettlemoyer, Luke and Morcos, Ari and Farhadi, Ali and Schmidt, Ludwig},
  booktitle={Advances in Neural Information Processing Systems (NeurIPS)},
  volume={36},
  year={2023}
}

@article{xi2023training,
  title={Training transformers with 4-bit integers},
  author={Xi, Haocheng and Li, Changhao and Chen, Jianfei and Zhu, Jun},
  journal={Advances in Neural Information Processing Systems},
  volume={36},
  pages={49146--49168},
  year={2023}
}

@article{chmiel2023accurate,
  title={Accurate neural training with 4-bit matrix multiplications at standard formats},
  author={Chmiel, Brian and Banner, Ron and Hoffer, Elad and Yaacov, Hilla Ben and Soudry, Daniel},
  journal={arXiv preprint arXiv:2112.10769},
  year={2021}
}

@inproceedings{salmon2011parallel,
  title={Parallel Random Numbers: As Easy as 1, 2, 3},
  author={Salmon, John K. and Moraes, Mark A. and Dror, Ron O. and Shaw, David E.},
  booktitle={Proceedings of the International Conference for High Performance Computing, Networking, Storage and Analysis (SC)},
  pages={1--12},
  year={2011},
  doi={10.1145/2063384.2063405}
}

@misc{nvidia2025ptx,
  author={{NVIDIA Corporation}},
  title={Parallel Thread Execution {ISA}, Release 9.0},
  howpublished={\url{https://docs.nvidia.com/cuda/parallel-thread-execution/}},
  note={Instruction \texttt{cvt}, stochastic rounding mode \texttt{.rs}},
  year={2025}
}

@misc{bradbury2018jax,
  title={JAX: composable transformations of Python+ NumPy programs},
  author={Bradbury, J},
  year={2018}
}

@inproceedings{dao2022flashattention,
  title={{FlashAttention}: Fast and Memory-Efficient Exact Attention with {IO}-Awareness},
  author={Dao, Tri and Fu, Daniel Y. and Ermon, Stefano and Rudra, Atri and R{\'e}, Christopher},
  booktitle={Advances in Neural Information Processing Systems (NeurIPS)},
  volume={35},
  year={2022}
}

@inproceedings{rajbhandari2020zero,
  title={{ZeRO}: Memory Optimizations Toward Training Trillion Parameter Models},
  author={Rajbhandari, Samyam and Rasley, Jeff and Ruwase, Olatunji and He, Yuxiong},
  booktitle={SC20: International Conference for High Performance Computing, Networking, Storage and Analysis},
  pages={1--16},
  year={2020},
  doi={10.1109/SC41405.2020.00024}
}

@article{shoeybi2019megatron,
  title={Megatron-lm: Training multi-billion parameter language models using model parallelism},
  author={Shoeybi, Mohammad and Patwary, Mostofa and Puri, Raul and LeGresley, Patrick and Casper, Jared and Catanzaro, Bryan},
  journal={arXiv preprint arXiv:1909.08053},
  year={2019}
}

\iffull
\appendices

\section{Proofs}
\label{app:proofs}

This section proves the results of Sections~\ref{subsec:analysis} and~\ref{subsec:s3r_opt}. It adds a result that the main text uses only in passing, the effect of sharing one quantized gradient (Proposition~\ref{prop:subspace_decorrelation}), and ends with numerical checks. Throughout, $\mathbf{G}$ is the gradient of the summed loss $\ell$, and $\bar{\mathbf{g}}_{\mathbf{W}} = \frac{1}{BS}\mathbf{X}^\top\widetilde{\mathbf{G}}$ estimates $\nabla_{\mathbf{W}}\mathcal{L}$ for the token-mean loss $\mathcal{L} = \ell/(BS)$.

\subsection{Proof of Lemma~\ref{lemma:unbiasedness} (Unbiasedness)}
\label{app:lemma1_proof}

Let $\mathcal{I} = \{1, \dots, B\} \times \{1, \dots, S\} \times \{1, \dots, H\}$ index a gradient tensor with $N = BSH$ entries. A topology $\Phi: [0, 1]^{d_0} \to [0, 1]^{B \times S \times H}$, with $d_0 \ll BSH$, maps a seed $\mathbf{u} \sim \mathcal{U}([0, 1]^{d_0})$ to a threshold tensor $\mathbf{R} = \Phi(\mathbf{u})$ whose entries $R_{b,s,h}$ are each marginally $\mathcal{U}(0, 1)$, with density $f_R(r) = 1$ on $[0, 1]$. The entries share seed components and are therefore correlated, but unbiasedness uses only the marginal law of each entry and the linearity of expectation.

\begin{proof}[Proof of Lemma~\ref{lemma:unbiasedness}]
\textit{Step 1: each entry is unbiased.} Fix $\mathbf{G} \in \mathbb{R}^{B \times S \times H}$, with step size $\Delta > 0$ and lower grid point $q(G_{b,s,h}) = \Delta \lfloor G_{b,s,h} / \Delta \rfloor$. The normalized fractional residual is
\begin{equation}
P_{b,s,h} = \frac{G_{b,s,h} - q(G_{b,s,h})}{\Delta} = \frac{G_{b,s,h}}{\Delta} - \left\lfloor \frac{G_{b,s,h}}{\Delta} \right\rfloor \in [0, 1).
\end{equation}
The rounded entry is $\widetilde{G}_{b,s,h} = q(G_{b,s,h}) + \Delta \mathbb{I}(P_{b,s,h} \ge R_{b,s,h})$, with error $E_{b,s,h} = \widetilde{G}_{b,s,h} - G_{b,s,h} = \Delta(\mathbb{I}(P_{b,s,h} \ge R_{b,s,h}) - P_{b,s,h})$. Because $R_{b,s,h}$ is marginally uniform, the indicator $I_{b,s,h} = \mathbb{I}(R_{b,s,h} \le P_{b,s,h})$ has mean
\begin{align}
\mathbb{E}_{\mathbf{u}}\left[I_{b,s,h}\right] &= \Pr\left(R_{b,s,h} \le P_{b,s,h}\right) \notag \\
&= \int_0^1 \mathbb{I}(r \le P_{b,s,h}) \, dr = \int_0^{P_{b,s,h}} 1 \, dr = P_{b,s,h}.
\end{align}
Hence $\mathbb{E}_{\mathbf{u}}[E_{b,s,h}] = \Delta(\mathbb{E}_{\mathbf{u}}[I_{b,s,h}] - P_{b,s,h}) = 0$ and $\mathbb{E}_{\mathbf{u}}[\widetilde{G}_{b,s,h}] = G_{b,s,h}$ for every $(b, s, h) \in \mathcal{I}$.

\textit{Step 2: linear contractions.} Let $T: \mathbb{R}^{B \times S \times H} \to \mathbb{R}^{d_1 \times d_2}$ be a deterministic linear operator with coefficients $\mathcal{K}_{i,j,b,s,h}$. The expectation of a finite linear combination is the linear combination of the expectations, whether or not the terms are correlated, so
\begin{align}
\mathbb{E}_{\mathbf{u}}\left[T(\widetilde{\mathbf{G}})_{i,j}\right] &= \mathbb{E}_{\mathbf{u}}\left[\sum_{b=1}^B \sum_{s=1}^S \sum_{h=1}^H \mathcal{K}_{i,j,b,s,h} \left(G_{b,s,h} + E_{b,s,h}\right)\right] \notag \\
&= \sum_{b=1}^B \sum_{s=1}^S \sum_{h=1}^H \mathcal{K}_{i,j,b,s,h} \left(G_{b,s,h} + \mathbb{E}_{\mathbf{u}}[E_{b,s,h}]\right) \notag \\
&= T(\mathbf{G})_{i,j}.
\end{align}

\textit{Step 3: the backward products.} The activation-gradient product $(\widetilde{\mathbf{G}}\mathbf{W}^\top)_{b,s,i} = \sum_{h=1}^{H_{\text{out}}} \widetilde{G}_{b,s,h} W_{i,h}$ and the weight-gradient product $(\mathbf{X}^\top\widetilde{\mathbf{G}})_{i,h} = \sum_{b=1}^B \sum_{s=1}^S X_{b,s,i} \widetilde{G}_{b,s,h}$ are of this form. One shared $\widetilde{\mathbf{G}} = \text{SR}(\mathbf{G}; \Phi(\mathbf{u}))$ therefore gives $\mathbb{E}_{\mathbf{u}}[\widetilde{\mathbf{G}}\mathbf{W}^\top] = \mathbf{G}\mathbf{W}^\top$ and $\mathbb{E}_{\mathbf{u}}[\mathbf{X}^\top\widetilde{\mathbf{G}}] = \mathbf{X}^\top\mathbf{G}$, and dividing the latter by $BS$ gives $\mathbb{E}_{\mathbf{u}}[\bar{\mathbf{g}}_{\mathbf{W}}] = \nabla_{\mathbf{W}}\mathcal{L}$. In a sparse MoE layer, the token-to-expert routing $M_{b,s,e} \in \{0, 1\}$ is fixed by the forward pass, and the expert gradient $T_{\text{MoE}, e}(\widetilde{\mathbf{G}})_{i,h} = \sum_{b=1}^B \sum_{s=1}^S M_{b,s,e} X_{b,s,i} \widetilde{G}_{b,s,h}$ is again linear in $\widetilde{\mathbf{G}}$ and hence unbiased.
\end{proof}

\subsection{Proposition~\ref{prop:subspace_decorrelation}: Effect of Sharing One Quantized Gradient}
\label{app:subspace_decorrelation}

S3R feeds one rounded gradient $\widetilde{\mathbf{G}} = \mathbf{G} + \mathbf{E}$ to both $\widehat{\nabla}_{\mathbf{X}} = \widetilde{\mathbf{G}}\mathbf{W}^\top$ and $\widehat{\nabla}_{\mathbf{W}} = \mathbf{X}^\top \widetilde{\mathbf{G}}$. The proposition below shows that the two products act on different dimensions of $\mathbf{E}$ and that sharing correlates their errors, but that this correlation enters none of the second moments used in our analysis. Claims~1--3 hold for fixed residuals.

\begin{proposition}[Row/Column Action and Effect of Sharing]
\label{prop:subspace_decorrelation}
Consider one example with error $\mathbf{E}_b \in \mathbb{R}^{S \times H_{\text{out}}}$, input activation $\mathbf{X}_b \in \mathbb{R}^{S \times H_{\text{in}}}$, and weight $\mathbf{W} \in \mathbb{R}^{H_{\text{in}} \times H_{\text{out}}}$, with reduced singular value decompositions $\mathbf{W} = \mathbf{U}_{\mathbf{W}}\mathbf{\Sigma}_{\mathbf{W}}\mathbf{V}_{\mathbf{W}}^\top$ ($\mathbf{U}_{\mathbf{W}} \in \mathbb{R}^{H_{\text{in}} \times r_W}, \mathbf{V}_{\mathbf{W}} \in \mathbb{R}^{H_{\text{out}} \times r_W}$) and $\mathbf{X}_b = \mathbf{U}_{\mathbf{X}}\mathbf{\Sigma}_{\mathbf{X}}\mathbf{V}_{\mathbf{X}}^\top$ ($\mathbf{U}_{\mathbf{X}} \in \mathbb{R}^{S \times r_X}, \mathbf{V}_{\mathbf{X}} \in \mathbb{R}^{H_{\text{in}} \times r_X}$). Let $\mathbf{\Delta}_{\mathbf{X}, b} = \mathbf{E}_b \mathbf{W}^\top \in \mathbb{R}^{S \times H_{\text{in}}}$ and $\mathbf{\Delta}_{\mathbf{W}, b} = \mathbf{X}_b^\top \mathbf{E}_b \in \mathbb{R}^{H_{\text{in}} \times H_{\text{out}}}$ be the resulting errors of the two backward products. Then:
\begin{enumerate}
    \item \textbf{Row vs.\ Column Action:} $\mathbf{\Delta}_{\mathbf{X}, b} = (\mathbf{E}_b \mathbf{V}_{\mathbf{W}})\mathbf{\Sigma}_{\mathbf{W}}\mathbf{U}_{\mathbf{W}}^\top$ acts on each \emph{row} of $\mathbf{E}_b$ (one token across channels) through the row space $\text{span}(\mathbf{V}_{\mathbf{W}}) \subseteq \mathbb{R}^{H_{\text{out}}}$ of $\mathbf{W}$, whereas $\mathbf{\Delta}_{\mathbf{W}, b} = \mathbf{V}_{\mathbf{X}}\mathbf{\Sigma}_{\mathbf{X}}(\mathbf{U}_{\mathbf{X}}^\top \mathbf{E}_b)$ acts on each \emph{column} of $\mathbf{E}_b$ (one channel across tokens) through the column space $\text{span}(\mathbf{U}_{\mathbf{X}}) \subseteq \mathbb{R}^S$ of $\mathbf{X}_b$.
    \item \textbf{Conditional Cross-Covariance:} Given $(\mathbf{W}, \mathbf{X}_b)$,
    \begin{align}
    &\mathbb{E}_{\mathbf{u}}\left[\mathbf{\Delta}_{\mathbf{X}, b, s, i} \, \mathbf{\Delta}_{\mathbf{W}, b, i', j}\right] \notag \\
    &\quad = \sum_{k=1}^{H_{\text{out}}} \sum_{s'=1}^{S} W_{i, k} X_{b, s', i'} \, \text{Cov}\left(E_{b, s, k}, E_{b, s', j}\right),
    \end{align}
    which is nonzero in general, even for independent thresholds. Sharing $\widetilde{\mathbf{G}}$ therefore correlates the activation- and weight-gradient errors, whereas separate thresholds for the two products would not.
    \item \textbf{Unchanged Second Moments:} Sharing does not change the marginal distribution of any entry of $\mathbf{E}$. Every entry of $\widehat{\nabla}_{\mathbf{X}}$ and of $\widehat{\nabla}_{\mathbf{W}}$ therefore has the same mean and variance as with separate thresholds of the same topology, and so does the expected squared norm of the full gradient. Covariances between coordinates, whether between $\widehat{\nabla}_{\mathbf{X}}$ and $\widehat{\nabla}_{\mathbf{W}}$ or across layers, do not enter the coordinatewise second moment $v_t$ of AdamW.
    \item \textbf{Inter-Layer Noise Adds in Variance:} If layer $\ell-1$ draws its thresholds independently of the noise arriving from layer $\ell$ (e.g., $\mathbf{u}^{(\ell)}_t = \text{FoldIn}(\text{key}_t, \ell)$), then $\mathbb{E}[\mathbf{E}_b^{(\ell-1)} \mid \mathbf{\Delta}_{\mathbf{X}, b}^{(\ell)}] = \mathbf{0}$. Propagated and local rounding noise are then uncorrelated, and their variances add. They are not independent, because $\mathbf{E}^{(\ell-1)}$ depends on $\mathbf{\Delta}_{\mathbf{X}}^{(\ell)}$ through the residuals and the scale.
\end{enumerate}
\end{proposition}

\begin{proof}
Each claim follows from writing the two errors as linear functions of $\mathbf{E}_b$.

For Claim~1, substituting $\mathbf{W}^\top = \mathbf{V}_{\mathbf{W}}\mathbf{\Sigma}_{\mathbf{W}}\mathbf{U}_{\mathbf{W}}^\top$ into $\mathbf{\Delta}_{\mathbf{X}, b} = \mathbf{E}_b \mathbf{W}^\top$ gives $\mathbf{\Delta}_{\mathbf{X}, b} = (\mathbf{E}_b \mathbf{V}_{\mathbf{W}})\mathbf{\Sigma}_{\mathbf{W}}\mathbf{U}_{\mathbf{W}}^\top$, which depends on $\mathbf{E}_b$ only through right multiplication by the orthonormal columns of $\mathbf{V}_{\mathbf{W}}$. Likewise, substituting $\mathbf{X}_b^\top = \mathbf{V}_{\mathbf{X}}\mathbf{\Sigma}_{\mathbf{X}}\mathbf{U}_{\mathbf{X}}^\top$ into $\mathbf{\Delta}_{\mathbf{W}, b} = \mathbf{X}_b^\top \mathbf{E}_b$ gives $\mathbf{\Delta}_{\mathbf{W}, b} = \mathbf{V}_{\mathbf{X}}\mathbf{\Sigma}_{\mathbf{X}}(\mathbf{U}_{\mathbf{X}}^\top \mathbf{E}_b)$, which depends on $\mathbf{E}_b$ only through left multiplication by $\mathbf{U}_{\mathbf{X}}^\top$.

For Claim~2, the $(s, i)$ entry of $\mathbf{\Delta}_{\mathbf{X}, b}$ times the $(i', j)$ entry of $\mathbf{\Delta}_{\mathbf{W}, b}$ is
\begin{align}
\mathbf{\Delta}_{\mathbf{X}, b, s, i} \, \mathbf{\Delta}_{\mathbf{W}, b, i', j} &= \Big(\sum_{k} E_{b, s, k} W_{i, k}\Big) \Big(\sum_{s'} X_{b, s', i'} E_{b, s', j}\Big) \notag \\
&= \sum_{k=1}^{H_{\text{out}}} \sum_{s'=1}^{S} W_{i, k} X_{b, s', i'} E_{b, s, k} E_{b, s', j}.
\end{align}
Holding $(\mathbf{W}, \mathbf{X}_b)$ fixed and taking the expectation over the thresholds, with $\mathbb{E}[\mathbf{E}] = \mathbf{0}$, gives
\begin{align}
&\mathbb{E}_{\mathbf{u}}\left[\mathbf{\Delta}_{\mathbf{X}, b, s, i} \, \mathbf{\Delta}_{\mathbf{W}, b, i', j} \;\middle|\; \mathbf{W}, \mathbf{X}_b\right] \notag \\
&\quad = \sum_{k=1}^{H_{\text{out}}} \sum_{s'=1}^{S} W_{i, k} X_{b, s', i'} \text{Cov}_{\mathbf{u}}\left(E_{b, s, k}, E_{b, s', j}\right).
\end{align}
With independent thresholds, only the term $k = j$, $s' = s$ survives, and the cross-covariance equals $W_{i,j} X_{b,s,i'} \text{Var}(E_{b,s,j})$. Since $\text{Var}(E_{b,s,j}) = \Delta^2 P_{b,s,j}(1 - P_{b,s,j}) > 0$ for $P_{b,s,j} \in (0, 1)$, it is nonzero whenever $W_{i,j} X_{b,s,i'} \neq 0$.

For Claim~3, each entry of either estimator is a fixed linear combination of entries of $\mathbf{E}$, so its mean and variance depend only on the joint law of $\mathbf{E}$, which sharing leaves unchanged. Covariances between entries of different estimators never enter a sum of squared entries.

For Claim~4, the upstream gradient of layer $\ell-1$ before quantization is $\mathbf{G}_b^{(\ell-1)} = (\mathbf{G}^{(\ell)}_b\mathbf{W}^{(\ell)\top} + \mathbf{\Delta}_{\mathbf{X}, b}^{(\ell)}) \odot \sigma'(\mathbf{Z}_b^{(\ell-1)})$. Given $\mathbf{\Delta}_{\mathbf{X}, b}^{(\ell)}$, the residuals and the scale of layer $\ell-1$ are fixed, and its thresholds are still marginally uniform. Lemma~\ref{lemma:unbiasedness} then gives $\mathbb{E}[\mathbf{E}_b^{(\ell-1)} \mid \mathbf{\Delta}_{\mathbf{X}, b}^{(\ell)}] = \mathbf{0}$, and by the tower property the products of errors from different layers have zero mean.
\end{proof}

\subsection{Proofs of Lemma~\ref{lemma:shared_offset}, Propositions~\ref{prop:coherent_dx} and~\ref{prop:two_way}, Theorem~\ref{thm:centering}, and Corollary~\ref{cor:structured_topologies}}
\label{app:two_way_proof}

\begin{proof}[Proof of Lemma~\ref{lemma:shared_offset}]
Write $G_i/\Delta = \lfloor G_i/\Delta\rfloor + P_i$. The kernel rounds element $i$ up exactly when $P_i + R_i \ge 1$, so $E_i = \Delta(I_i - P_i)$ with $I_i = \mathbb{I}(R_i \ge 1 - P_i)$, and $\Pr(I_i = 1) = P_i$ for a uniform $R_i$. If $R_i = R_k = r$, both round up exactly when $r \ge 1 - \min(P_i, P_k)$, so $\mathbb{E}[I_iI_k] = \min(P_i, P_k)$ and $\text{Cov}(E_i, E_k) = \Delta^2(\min(P_i, P_k) - P_iP_k)$. With $a = \min(P_i, P_k)$ and $b = \max(P_i, P_k)$, this equals $\Delta^2 a(1 - b) \in [0, \Delta^2 a(1 - a)] \subseteq [0, \Delta^2/4]$; it is zero when $a = 0$ and tends to zero as $b \to 1$. If $R_i$ and $R_k$ are independent, so are $E_i$ and $E_k$. For residuals spread like a uniform sample, the average variance and the average pairwise covariance approach $\Delta^2\int_0^1 p(1-p)\,dp = \Delta^2/6$ and $\Delta^2\int_0^1\!\int_0^1(\min(p, q) - pq)\,dp\,dq = \Delta^2(\frac{1}{3} - \frac{1}{4}) = \Delta^2/12$.
\end{proof}

\begin{proof}[Proof of Proposition~\ref{prop:coherent_dx}]
For the first statement, entry $(i, j)$ of the weight-gradient error is $\frac{1}{BS}\sum_{b,s} X_{b,s,i}E_{b,s,j}$. Under $\Phi_{\text{2D}}$, errors of different tokens use independent thresholds and are therefore independent, so its variance is $\frac{1}{(BS)^2}\sum_{b,s} X_{b,s,i}^2\,\text{Var}(E_{b,s,j})$. Each $\text{Var}(E_{b,s,j}) = \Delta^2P_{b,s,j}(1 - P_{b,s,j})$ depends only on the marginal law of its threshold, so independent SR gives the same sum. For the second statement, the variance of $\sum_j W_{ij}E_{b,s,j}$ is $\sum_j W_{ij}^2\,\text{Var}(E_{b,s,j}) + \sum_{j \neq k} W_{ij}W_{ik}\,\text{Cov}(E_{b,s,j}, E_{b,s,k})$; independent SR has the same first sum and a zero second sum. With equal covariances $c$, the second sum is $c(\sum_j W_{ij})^2 - c\sum_j W_{ij}^2$.

For the third statement, a correction of the weight gradient changes $\bar{\mathbf{g}}_{\mathbf{W}}$ but not $\widehat{\nabla}_{\mathbf{X}}$. Let $\mathcal{F}$ collect the forward pass and all randomness of the layers above. Let $\mathbf{A}$ be the contribution of the arriving activation-gradient error to a linear function of the lower layer's gradient, for example its weight gradient, let $\mathbf{Z}$ be the contribution of that layer's own rounding error, and let $\mathbf{C}$ be a correction linear in that error. The lower layer draws its thresholds independently of $\mathcal{F}$, so Lemma~\ref{lemma:unbiasedness}, applied given $\mathcal{F}$, yields $\mathbb{E}[\mathbf{Z} \mid \mathcal{F}] = \mathbb{E}[\mathbf{C} \mid \mathcal{F}] = \mathbf{0}$ (see also Proposition~\ref{prop:subspace_decorrelation}, Claim~4), whereas $\mathbf{A}$ is determined by $\mathcal{F}$. Hence $\mathbb{E}\langle\mathbf{A}, \mathbf{Z} - \mathbf{C}\rangle = 0$ and $\mathbb{E}\|\mathbf{A} + \mathbf{Z} - \mathbf{C}\|_F^2 = \mathbb{E}\|\mathbf{A}\|_F^2 + \mathbb{E}\|\mathbf{Z} - \mathbf{C}\|_F^2 \ge \mathbb{E}\|\mathbf{A}\|_F^2$. A correction built from the upper layer's error, such as~(\ref{eq:s3r_opt_dx}), is outside this class: it carries that error downward.
\end{proof}

\begin{proof}[Proof of Theorem~\ref{thm:centering}]
Let $C_{ik} = \text{Cov}(X_i, X_k)$ and $a_i = \text{Cov}(X_i, \bar{X}) = \frac{1}{n}\sum_k C_{ik}$. Then $\frac{1}{n}\sum_i a_i = \text{Var}(\bar{X}) = \frac{1}{n^2}\sum_{i,k} C_{ik} = \frac{1}{n}\bar{\sigma}^2 + \frac{n-1}{n}\bar{c}$. By bilinearity, $\text{Var}(Y_i) = \text{Var}(X_i) - 2a_i + \text{Var}(\bar{X})$; averaging over $i$ gives the second statement, and since the left side is nonnegative, $\bar{c} \le \bar{\sigma}^2$. For the first statement, $\sum_i Y_i = 0$ on every realization, so $0 = \text{Var}(\sum_i Y_i) = \sum_i \text{Var}(Y_i) + \sum_{i \neq k}\text{Cov}(Y_i, Y_k)$. The average pairwise covariance of the $Y_i$ is therefore $-\frac{1}{n-1}$ times their average variance, which by the second statement equals $-\frac{1}{n}(\bar{\sigma}^2 - \bar{c}) = \bar{c} - \text{Var}(\bar{X})$. For the third statement, $\text{Cov}(Y_i, Y_k) = C_{ik} - a_i - a_k + \text{Var}(\bar{X})$, and equal covariances give $a_i = \text{Var}(\bar{X})$ for every $i$, so $\text{Cov}(Y_i, Y_k) = c - \text{Var}(\bar{X}) = -\frac{1}{n}(\sigma^2 - c)$. Moreover, $\sum_i w_iY_i = \sum_i w_iX_i - n\bar{w}\bar{X} = \sum_i (w_i - \bar{w})X_i$. The covariance matrix of the $X_i$ is $(\sigma^2 - c)\mathbf{I} + c\mathbf{1}\mathbf{1}^\top$, and the weights $w_i - \bar{w}$ sum to zero, so the variance is $(\sigma^2 - c)\sum_i(w_i - \bar{w})^2$.

The criterion for a single element follows from $\text{Var}(Y_i) = \text{Var}(X_i) - 2a_i + \text{Var}(\bar{X})$. In the example of the main text, $\bar{X} = Z/3$, $Y_1 = -\frac{4}{3}Z$, and $Y_2 = Y_3 = \frac{2}{3}Z$, so $\text{Var}(Y_1) = \frac{16}{9}\text{Var}(Z)$, and the average variance is $\frac{8}{9}\text{Var}(Z) = (1 - \frac{1}{9})\text{Var}(Z)$, as the second statement requires with $\text{Var}(\bar{X}) = \frac{1}{9}\text{Var}(Z)$.
\end{proof}

\begin{proof}[Proof of Corollary~\ref{cor:structured_topologies}]
Under $\Phi_{\text{2D}}$, the channels of token $(b, s)$ share $u_{b,s}$, so the first statement is the first statement of Theorem~\ref{thm:centering} with $n = D$, and $\bar{c}_{b,s} \ge 0$ by Lemma~\ref{lemma:shared_offset}. Entries of channel $j$ in different tokens use independent thresholds, so their covariances vanish by the same lemma, and the second statement of Theorem~\ref{thm:centering} with $n = S$ and $\bar{c} = 0$ gives a reduction of $\frac{1}{S}\bar{\sigma}^2_{b,j}$.
\end{proof}

\begin{proof}[Proof of Proposition~\ref{prop:two_way}]
All four claims follow from two facts: $\mathbf{E}^{\text{str}}_b$ is linear in $\mathbf{E}_b$, and the remaining error $\mathbf{N}_b$ has zero row and column sums.

For Claim~1, both subtracted terms are linear in $\mathbf{E}$, with coefficients ($\mathbf{W}$, $\mathbf{X}_b$, and fixed averaging vectors) determined by the forward pass, and $\mathbb{E}[\mathbf{E}] = \mathbf{0}$ by Lemma~\ref{lemma:unbiasedness}. The corrected estimators therefore keep the means $\mathbf{G}\mathbf{W}^\top$ and $\nabla_{\mathbf{W}}\mathcal{L}$.

For Claim~2, $\mathbf{E}^{\text{str}}_b \in \mathcal{S}$ by construction. Using $\mathbf{1}_S^\top\boldsymbol{\varepsilon}_b = Sm_b$ and $\bar{\mathbf{e}}_b^\top\mathbf{1}_D = Dm_b$, we get $\mathbf{N}_b^\top\mathbf{1}_S = S\bar{\mathbf{e}}_b - S\bar{\mathbf{e}}_b - \mathbf{1}_D(Sm_b - Sm_b) = \mathbf{0}$ and $\mathbf{N}_b\mathbf{1}_D = D\boldsymbol{\varepsilon}_b - Dm_b\mathbf{1}_S - D(\boldsymbol{\varepsilon}_b - m_b\mathbf{1}_S) = \mathbf{0}$. Hence $\langle\mathbf{N}_b, \mathbf{1}_S\mathbf{a}^\top + \mathbf{c}\mathbf{1}_D^\top\rangle = \mathbf{a}^\top\mathbf{N}_b^\top\mathbf{1}_S + \mathbf{c}^\top\mathbf{N}_b\mathbf{1}_D = 0$. So $\mathbf{N}_b \perp \mathcal{S}$, $\mathbf{E}^{\text{str}}_b$ is the orthogonal projection, and $\|\mathbf{E}_b\|_F^2 = \|\mathbf{E}^{\text{str}}_b\|_F^2 + \|\mathbf{N}_b\|_F^2$.

For Claim~3, $\mathbf{N}_b\mathbf{1}_D = \mathbf{0}$ states that the $D$ remaining errors of every token sum to zero, and $\mathbf{N}_b^\top\mathbf{1}_S = \mathbf{0}$ that the $S$ remaining errors of every channel do. For such a group, the argument in the proof of Theorem~\ref{thm:centering} applies verbatim: the sum has zero variance, so the average pairwise covariance is $-\frac{1}{n-1}$ times the average variance. Moreover, $\mathbf{E}^{\text{str}}_b\mathbf{W}^\top = \mathbf{1}_S(\mathbf{W}\bar{\mathbf{e}}_b)^\top + (\boldsymbol{\varepsilon}_b - m_b\mathbf{1}_S)(\mathbf{W}\mathbf{1}_D)^\top$, so subtracting it removes from $\mathbf{E}_b\mathbf{W}^\top$ the term $\boldsymbol{\varepsilon}_b(\mathbf{W}\mathbf{1}_D)^\top$, whose row $s$ is the shared term $\varepsilon_{b,s}\mathbf{W}\mathbf{1}_D$ of token $s$.

For Claim~4, Claim~2 gives $\langle\mathbf{N}_b, \mathbf{E}^{\text{str}}_b\rangle = 0$ on every realization, so $\|\mathbf{E}_b\|_F^2 = \|\mathbf{N}_b\|_F^2 + \|\mathbf{E}^{\text{str}}_b\|_F^2$. Every entry of $\mathbf{E}_b$, $\mathbf{N}_b$, and $\mathbf{E}^{\text{str}}_b$ has zero mean by Lemma~\ref{lemma:unbiasedness} and linearity, so taking expectations gives the identity.
\end{proof}

\subsection{Numerical Validation}
\label{app:numerical_validation}

We checked Lemma~\ref{lemma:shared_offset} and Theorem~\ref{thm:centering} with exact second moments for one token with $D = 256$ channels, $\Delta = 1$, and one shared threshold, rounded as in the kernel. For fixed residuals, Lemma~\ref{lemma:shared_offset} gives the covariance matrix of the errors in closed form, so every entry of Table~\ref{tab:centering_check} is exact. A Monte Carlo run with 20{,}000 thresholds reproduces the variance of one output before and after centering, 62.7 against 62.6 exact and 21.2 against 21.2. The residuals are a fixed uniform sample; a sample of $U^4$ with $U$ uniform, which places most of them near 0, as for small gradient entries; or all equal to $\frac{1}{2}$. The 128 activation-gradient outputs use weight rows with i.i.d.\ $\mathcal{N}(0, 1)$ or $\mathcal{N}(0.05, 1)$ entries; the latter adds a common component, which enlarges $\mathbf{W}\mathbf{1}_D$. Centering subtracts the token mean, as S3R-EC does for one token.

\begin{table*}[t]
\centering
\caption{Exact Second Moments of Per-Token Centering for One Token ($D = 256$, $\Delta = 1$, One Shared Threshold, 128 Activation-Gradient Outputs). $\bar{\sigma}^2$: Average Variance of the Entries. $\bar{c}$: Their Average Pairwise Covariance before and after Centering. The Last Two Columns Divide the Summed Output Variance after Centering by That before It and by That of Independent SR}
\label{tab:centering_check}
\resizebox{0.9\textwidth}{!}{%
\begin{tabular}{@{} l l c c c c c c c c @{}}
\toprule
Residuals & $\mathbf{W}$ entries & $\bar{\sigma}^2$ & $\bar{c}$ before & $\bar{c}$ after & Entry var.\ after & Noisier entries & Noisier outputs & vs.\ before & vs.\ indep.\ SR \\
\midrule
Uniform & $\mathcal{N}(0, 1)$ & 0.167 & 0.083 & $-3.3{\times}10^{-4}$ & 0.083 & 46/256 & 22/128 & 0.468 & 0.541 \\
Uniform & $\mathcal{N}(0.05, 1)$ & 0.167 & 0.083 & $-3.3{\times}10^{-4}$ & 0.083 & 46/256 & 25/128 & 0.371 & 0.514 \\
$U^4$ & $\mathcal{N}(0, 1)$ & 0.093 & 0.028 & $-2.5{\times}10^{-4}$ & 0.064 & 109/256 & 34/128 & 0.600 & 0.698 \\
$U^4$ & $\mathcal{N}(0.05, 1)$ & 0.093 & 0.028 & $-2.5{\times}10^{-4}$ & 0.064 & 109/256 & 34/128 & 0.602 & 0.774 \\
All $\frac{1}{2}$ & $\mathcal{N}(0, 1)$ & 0.250 & 0.250 & 0 & 0.000 & 0/256 & 0/128 & 0.000 & 0.000 \\
All $\frac{1}{2}$ & $\mathcal{N}(0.05, 1)$ & 0.250 & 0.250 & 0 & 0.000 & 0/256 & 0/128 & 0.000 & 0.000 \\
\bottomrule
\end{tabular}%
}
\vspace{-6pt}
\end{table*}

In every case, centering turns the average pairwise covariance of the entries from 0.028--0.250 into a value between $-3.3 \times 10^{-4}$ and 0, and Theorem~\ref{thm:centering} holds to within $10^{-16}$. With all residuals equal, every element rounds in the same direction, and centering removes the error completely. Otherwise, the covariances differ between elements, and 18--43\% of the entries and 17--27\% of the outputs become noisier, the worst output by $3.4\times$. The summed output variance nevertheless falls to 0.37--0.60 of its uncentered value and to 0.51--0.77 of that of independent SR, which has the same per-element variances but no covariances.

\section{Algorithm Specifications and XLA Lowering on TPU v5p}
\label{app:algorithms_xla}
\begin{table*}[t!]
\centering
\caption{Algorithm 1: S3R Backward Pass of One Quantized Linear Layer with the Two-Way S3R-EC Correction (Plain JAX Compiled by XLA)}
\label{tab:algorithms}
\small
\begin{tabular}{@{} r @{\hspace{6pt}} p{0.92\textwidth} @{}}
\toprule
\multicolumn{2}{@{}l@{}}{\textbf{Algorithm 1:} S3R backward quantization and INT8 GEMMs with the two-way S3R-EC correction} \\
\midrule
\multicolumn{2}{@{}p{0.96\textwidth}@{}}{\textbf{Inputs:} upstream gradient $\mathbf{G}_t \in \text{BF16}^{[B,S,D_{\text{out}}]}$; forward operands $\widehat{\mathbf{X}}_t \in \text{S8}^{[B,S,H_{\text{in}}]}$ with per-token scales $\mathbf{s}_x \in \text{FP32}^{[B,S]}$ and $\widehat{\mathbf{W}}_t \in \text{S8}^{[H_{\text{in}}, D_{\text{out}}]}$ with per-output-channel scales $\mathbf{s}_w \in \text{FP32}^{[D_{\text{out}}]}$; PRNG key $\kappa_t$; broadcast $\Phi \in \{\Phi_{\text{2D}}\ \text{(default)}, \Phi_{\text{bcast}}\}$.} \\
\multicolumn{2}{@{}p{0.96\textwidth}@{}}{\textbf{Outputs:} corrected activation gradient $\widehat{\nabla}_{\mathbf{X},t} \in \text{BF16}^{[B,S,H_{\text{in}}]}$; corrected token-mean weight gradient $\bar{\mathbf{g}}_{\mathbf{W},t} \in \text{FP32}^{[H_{\text{in}}, D_{\text{out}}]}$ for the unchanged AdamW update. Divisions by $\mathbf{s}_x$ and $\mathbf{s}_w$ rescale rows and columns.} \\
\midrule
1: & \textbf{// Step 1: compact seed, four thresholds per random word} \\
2: & $N_{\text{u32}} \leftarrow \lceil BS/4 \rceil$ for $\Phi_{\text{2D}}$ or $\lceil B/4 \rceil$ for $\Phi_{\text{bcast}}$; $\quad \mathbf{z} \leftarrow \texttt{Threefry2x32}(\kappa_t, \text{shape}=[N_{\text{u32}}])$ \\
3: & \textbf{for} $m \in \{0, 1, 2, 3\}$ \textbf{do} \\
4: & $\quad \mathbf{b}_m \leftarrow (\mathbf{z} \gg 8m) \;\texttt{\&}\; \texttt{0xFF}$ \hfill {\scriptsize $\triangleright$ four 8-bit slices per word} \\
5: & $\quad \mathbf{u}_m \leftarrow \texttt{Bitcast\_BF16}(\texttt{0x3F80} \;\texttt{|}\; \mathbf{b}_m) - 1.0_{\text{bf16}} + 2^{-8}$ \hfill {\scriptsize $\triangleright$ grid $\{(2k+1)/256\}_{k=0}^{127}$} \\
6: & \textbf{end for}; $\quad \mathbf{u}_t \leftarrow \texttt{Reshape}(\{\mathbf{u}_m\}_{m=0}^{3})$ to $[B, S, 1]$ ($\Phi_{\text{2D}}$) or $[B, 1, 1]$ ($\Phi_{\text{bcast}}$) \\
7: & \textbf{// Step 2: quantize the folded gradient once ($\alpha_{\text{share}} = 1/2$)} \\
8: & $\mathbf{G}'_t \leftarrow \mathrm{diag}(\mathbf{s}_x)\,\mathbf{G}_t\,\mathrm{diag}(\mathbf{s}_w)$ (FP32); $\quad \Delta_t \leftarrow \max_{b,s,h} |G'_{t,b,s,h}| / 127$ \hfill {\scriptsize $\triangleright$ Eq.~(\ref{eq:scale_fold})} \\
9: & $\mathbf{R}_t \leftarrow \texttt{Broadcast}_{\Phi}(\mathbf{u}_t, [B, S, D_{\text{out}}])$ \hfill {\scriptsize $\triangleright$ fused into Line 10; never materialized} \\
10: & $\mathbf{S}_t \leftarrow \mathbf{G}'_t / \Delta_t$; $\quad \mathbf{Q}_t \leftarrow \texttt{Clip}\big(\lfloor \mathbf{S}_t + \mathbf{R}_t \rfloor, -127, 127\big) \in \text{S8}^{[B,S,D_{\text{out}}]}$ \hfill {\scriptsize $\triangleright$ written to HBM once, 1 byte per element} \\
11: & \textbf{// Step 3: two INT8 GEMMs read the same $\mathbf{Q}_t$} \\
12: & $\widehat{\nabla}_{\mathbf{X}, t} \leftarrow \Delta_t \cdot \texttt{DotGeneral\_S8}(\mathbf{Q}_t, \widehat{\mathbf{W}}_t^\top) / \mathbf{s}_x$ \\
13: & $\bar{\mathbf{g}}_{\mathbf{W}, t} \leftarrow \frac{\Delta_t}{B S} \cdot \texttt{DotGeneral\_S8}(\widehat{\mathbf{X}}_t^\top, \mathbf{Q}_t) / \mathbf{s}_w$ \hfill {\scriptsize $\triangleright$ token mean, FP32} \\
14: & \textbf{// Step 4: two-way S3R-EC on the folded error} \\
15: & $\mathbf{E}_b \leftarrow \Delta_t \mathbf{Q}_{t,b} - \mathbf{G}'_{t,b}$; $\; \bar{\mathbf{e}}_b, \boldsymbol{\varepsilon}_b, m_b$ as in (\ref{eq:estr}); $\; \bar{\mathbf{x}}_b \leftarrow \frac{1}{S}\widehat{\mathbf{X}}_{t,b}^\top\mathbf{1}_S$ \hfill {\scriptsize $\triangleright$ f32 reductions in the quantizer fusion} \\
16: & $\widehat{\nabla}_{\mathbf{X}_b, t} \leftarrow \widehat{\nabla}_{\mathbf{X}_b, t} - \big[\mathbf{1}_S(\widehat{\mathbf{W}}_t\bar{\mathbf{e}}_b)^\top + (\boldsymbol{\varepsilon}_b - m_b\mathbf{1}_S)(\widehat{\mathbf{W}}_t\mathbf{1}_D)^\top\big] / \mathbf{s}_x$ \hfill {\scriptsize $\triangleright$ Eq.~(\ref{eq:s3r_opt_dx})} \\
17: & $\bar{\mathbf{g}}_{\mathbf{W}, t} \leftarrow \bar{\mathbf{g}}_{\mathbf{W}, t} - \frac{1}{BS}\sum_b \big[S\bar{\mathbf{x}}_b \bar{\mathbf{e}}_b^\top + \widehat{\mathbf{X}}_{t,b}^\top(\boldsymbol{\varepsilon}_b - m_b\mathbf{1}_S)\mathbf{1}_D^\top\big] / \mathbf{s}_w$ \hfill {\scriptsize $\triangleright$ Eq.~(\ref{eq:s3r_opt_gc})} \\
18: & \textbf{return} $(\widehat{\nabla}_{\mathbf{X}, t}, \, \bar{\mathbf{g}}_{\mathbf{W}, t})$ \\
\bottomrule
\end{tabular}
\vspace{-6pt}
\end{table*}

\begin{figure*}[t]
    \centering
    \includegraphics[width=0.96\textwidth]{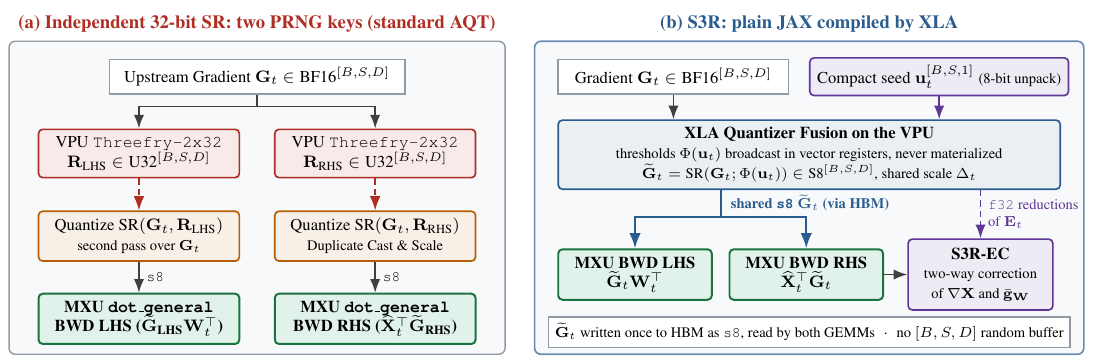}
    \vspace{-4pt}
    \caption{XLA lowering of backward quantization on TPU v5p (schematic). \textbf{(a)} Independent 32-bit SR with two PRNG keys (standard AQT): $\mathbf{G}_t$ is quantized twice, once per backward product, each time with full-size thresholds. Partitionable Threefry evaluates one hash per element inside the quantizer; non-partitionable Threefry materializes $[B, S, D]$ \texttt{u32} buffers. \textbf{(b)} S3R, written in plain JAX and compiled by XLA: the thresholds of a compact seed are broadcast in vector registers inside the XLA quantizer fusion and never materialized, and $\widetilde{\mathbf{G}}_t$, stored as S8 codes in $\text{S8}^{[B,S,D]}$, is written to HBM once and read by both backward GEMMs. S3R-EC uses small \texttt{f32} reductions of the rounding error emitted by the quantizer fusion. There is no custom kernel.}
    \label{fig:xla_hlo_fusion_comparison}
    \vspace{-4pt}
\end{figure*}

Table~\ref{tab:algorithms} specifies the S3R backward pass of one quantized linear layer in our JAX/MaxText implementation, which builds on the AQT (Accurate Quantized Training) library. It is plain JAX compiled by XLA, with no custom kernel, and the optimizer is AdamW applied to the corrected gradients. Fig.~\ref{fig:xla_hlo_fusion_comparison} shows how XLA lowers backward quantization on Google Cloud TPU v5p. This section explains that lowering and reports the vector instruction counts of the compiled step.

\subsection{Threshold Generation and In-Register Broadcast}
On TPU v5p, each TensorCore has four $128 \times 128$ systolic Matrix Multiply Units (MXUs) and one Vector Processing Unit (VPU). The VPU operates on vector registers of $8 \times 128$ 32-bit lanes and also runs the random-number generator. \texttt{Threefry-2x32} runs 20 rounds of 32-bit additions, bit rotations and XORs on pairs of \texttt{u32} words, which costs 127--134 vector ALU operations per 32-bit word in our compiled code (counted below). Our independent-SR baseline draws f32 thresholds with \texttt{jax.random.uniform}; with the default partitionable implementation, every element receives its own Threefry evaluation. The legacy non-partitionable mode packs four 8-bit values into each word, but it cannot be partitioned across devices without redundant computation or communication.

S3R follows the three steps of Algorithm~1. First, Lines 2--6 use all four bytes of each Threefry word and shift the threshold grid by half a step. Second, Threefry runs only on the seed: $\lceil BS/4 \rceil$ words for $\Phi_{\text{2D}}$, whose seed shape is $[B, S, 1]$, and $\lceil B/4 \rceil$ words for $\Phi_{\text{bcast}}$, whose seed shape is $[B, 1, 1]$. Third, XLA fuses the \texttt{broadcast\_in\_dim} of the seed into the \texttt{add} that computes $\lfloor \mathbf{S}_t + \mathbf{R}_t \rfloor$ (Lines 9--10), so the thresholds are broadcast as the operand is read into vector registers and the full tensor $\mathbf{R}_t$ is never materialized. The compiled program confirms this: the S3R step contains no PRNG fusion that generates more than $2^{20}$ elements, and its temporary memory is within 1\,MiB of RTNE with $\Phi_{\text{2D}}$ and within 37\,MiB with $\Phi_{\text{bcast}}$.

\subsection{Four Thresholds per Random Word}
\label{app:bit_unpack}
A \texttt{bf16} threshold uses only 7 random bits, yet partitionable Threefry hashes once per element. Packing several narrow draws into one word is known practice: JAX's legacy implementation does so for 8- and 16-bit draws \cite{bradbury2018jax}, and NVIDIA's stochastic-rounding conversion takes one 32-bit random operand for four FP4 values \cite{nvidia2025ptx}. SR hardware \cite{mikaitis2021stochastic} and SR with few random bits \cite{elarar2025probabilistic, fitzgibbon2025stochastic} have also been studied. S3R splits each word $\mathbf{z}$ into four slices $\mathbf{b}_m = (\mathbf{z} \gg 8m) \;\texttt{\&}\; \texttt{0xFF}$, $m \in \{0,1,2,3\}$, ORs each slice into the mantissa of $1.0$ (bit pattern \texttt{0x3F80}) and subtracts $1$:
\begin{equation}
\label{eq:ieee_bittrick}
\mathbf{u}_m = \texttt{Bitcast\_BF16}\left(\texttt{0x3F80} \;\texttt{|}\; \mathbf{b}_m\right) - 1.0_{\text{bf16}} \in [0, 1).
\end{equation}
Unpacking cuts Threefry evaluations by $4\times$, a small saving once thresholds are broadcast (Section~\ref{subsec:broadcast}).

Only the low seven bits of each slice reach the mantissa. The resulting grid $\{k/128\}_{k=0}^{127}$ would bias rounding in one direction by up to $\Delta/128$; with $\mathbf{u}_m + 2^{-8}$, an element rounds up with probability $P$ rounded to the nearest multiple of $2^{-7}$, a bias of at most $\Delta/256$ that averages out only over uniformly spread residuals. Since all thresholds lie in $[2^{-8}, 1 - 2^{-8}]$, elements with $|G| < \Delta/256$ round to zero whatever their sign, a dead zone $128\times$ narrower than under RTNE ($|G| < \Delta/2$). Lemma~\ref{lemma:unbiasedness} thus holds up to this discretization.

\subsection{Quantizing the Upstream Gradient Once}
In standard quantized backpropagation with AQT, the \texttt{dot\_general} that computes $\nabla_{\mathbf{X}}$ and the one that computes $\nabla_{\mathbf{W}}$ have separate quantization rules with distinct PRNG keys, $\kappa_{\text{LHS}} \neq \kappa_{\text{RHS}}$ (Fig.~\ref{fig:xla_hlo_fusion_comparison}a). Because the keys differ, XLA's common-subexpression elimination cannot merge the two quantizations of $\mathbf{G}_t$. The compiled step then draws two full sets of thresholds, computes the scale with two max-reductions, and runs the \texttt{bf16 $\to$ s8} rounding twice. S3R quantizes $\mathbf{G}'_t$ once with one key (Fig.~\ref{fig:xla_hlo_fusion_comparison}b; Algorithm~1, Lines 8--13). Its S8 output is written to HBM once and read by both GEMMs, which halves the quantization work ($\alpha_{\text{share}} = 1/2$). The two-way correction reuses $\Delta_t$ and the rounding error of this fusion (Lines 15--17).

\subsection{Vector Instruction Counts of the Compiled Step}
We count vector ALU (VALU) operations in the compiled $W_1$ step. The TPU compiler packs VALU operations into very-long-instruction-word (VLIW) instructions; one instruction word issues per cycle and can hold several VALU operations. We take the static count of each basic block from the final instruction words and multiply it by the block's resolved loop trip counts. One VALU operation processes one vector register of $8 \times 128$ 32-bit lanes. These are static schedule counts, not hardware counter readings. On a chain of dependent int32 operations, one HLO elementwise operation compiles to 1.078 VALU operations. We normalize step-level counts per gradient element, i.e., per element of the upstream gradients that receives a threshold ($N = BS(2D_{\text{ffn}} + H)$ per FFN step). To compare a measured extra device time with these counts, we multiply it by the attained VPU rate $\pi_{\text{VPU}} = 14.34 \times 10^{12}$ int32 elementwise operations per second, measured with a microbenchmark, and divide by $N$. The result is the number of VALU operations per gradient element that would take the same time.
\begin{enumerate}
    \item \textbf{Independent 32-bit SR:} With partitionable Threefry, the quantizer fusions of the upstream gradients issue 139.0 VALU operations per output element for $\nabla\mathbf{Z}_{\text{gate}}$ and $\nabla\mathbf{Z}_{\text{up}}$ and 142.0 for $\nabla\mathbf{Y}$, against 12.3 and 8.1 for the corresponding RTNE quantizers. The difference, about 127--134 operations, is the Threefry hash of one 32-bit word. Over the whole step, independent SR issues 127.6 more VALU operations per gradient element than RTNE.
    \item \textbf{S3R:} With $\Phi_{\text{2D}}$, the same quantizer fusions issue 14.5 and 11.0 VALU operations per element (14.5 and 11.1 with $\Phi_{\text{bcast}}$); the extra 2.2 and 2.9 operations over RTNE add the broadcast threshold. Over the whole step, S3R issues 2.4 more VALU operations per gradient element than RTNE with $\Phi_{\text{2D}}$ and 1.1 fewer with $\Phi_{\text{bcast}}$. Its only Threefry work acts on the seed. At $W_7$, this work amounts to $4.1 \times 10^{6}$ bitwise operations per step, summed over elements, with $\Phi_{\text{2D}}$ and 2460 with $\Phi_{\text{bcast}}$, against $4.49 \times 10^{10}$ for independent SR. On $\Phi_{\text{bcast}}$ we also counted a dW-only correction, which adds 8.2 VALU operations per gradient element.
    \item \textbf{S3R-EC on $\Phi_{\text{2D}}$:} RTNE, S3R and both corrections were compiled with one compiler and library revision, and the RTNE and S3R counts equal those above. The two-way correction adds 16.6 VALU operations per gradient element over S3R (18.9 over RTNE), and the dX-only correction adds 10.6 (13.0). The quantizer fusions of the upstream gradients gain small \texttt{f32} outputs for the per-token and per-channel reductions of $\mathbf{E}$, each weight quantizer gains one \texttt{f32} reduction, and the number of fusions grows from 76 to 108 (two-way) and 98 (dX-only). The added instruction words, $1.80 \times 10^{5}$ and $1.55 \times 10^{5}$ per core, take 0.095 and 0.082\,ms at the issue rate of 1.89\,GHz implied by the int32 operation chain. This is 91\% and 69\% of the measured $W_1$ increments over S3R (0.105 and 0.118\,ms, mean of two runs), so the rest of the dX-only increment is not explained by issued instructions. The RTNE and S3R compilations contain one data-dependent region whose trip count we could not resolve; it can add at most 0.31 VALU operations per gradient element, and the correction compilations have no such region. Both corrections stay far below the 127.6 VALU operations per gradient element that independent Threefry SR adds.
\end{enumerate}

\section{Experimental Setup and Hyperparameters}
\label{app:experimental_setup}

Table~\ref{tab:hyperparam_specs} summarizes the settings of the three workloads.

\textbf{Convergence runs.} We train on TFDS \texttt{c4/en} 3.0.1 with the default SentencePiece tokenizer of MaxText. Documents are packed into one stream as \texttt{[BOS]} document \texttt{[EOS]} and cut into consecutive non-overlapping windows, so every training token is seen at most once. Each seed starts at a different offset of the stream, so seeds see disjoint data, and all configurations see the same data for a given seed. Each configuration sweeps the peak learning rate within $\{0.75, 1.5, 3, 6, 12, 24\} \times 10^{-4}$. The corrected configurations cover $1.5$--$12 \times 10^{-4}$, with at least one point on each side of their optimum, and the sweep of uncorrected S3R with $\Phi_{\text{bcast}}$ extends down to $3.75 \times 10^{-5}$. The validation loss is measured on 16 fixed batches of $8 \times 2049$ tokens (262{,}272 tokens from the first 548 validation documents), with each configuration's own forward pass. The correction products use the default matrix-multiplication precision (\texttt{bf16} inputs, \texttt{f32} accumulation), here and in the timing runs.

\textbf{Metric definitions.} For the FFN block, the profiler splits the device time into categories. We report the time of the fusions that contain a GEMM (MXU-category time) and the share of busy time spent in PRNG fusions. The SR share of $v_t$ is the fraction of $\sum\hat{v}_t$ explained by the measured SR variance, late in training. For the per-layer diagnostics (Table~\ref{tab:grad_diag}), we take checkpoints of a BF16 run and compute each configuration's gradients on fixed training batches twice, with independent thresholds, to obtain $\mathbf{d}_1$ and $\mathbf{d}_2$. The stochastic error $(\mathbf{d}_1 - \mathbf{d}_2)/\sqrt{2}$ has the covariance of the SR error of one draw. We report its Frobenius norm relative to the BF16 gradient of the same batch.

\textbf{FFN-block measurements.} Table~\ref{tab:ffn_detail} lists the device times behind Table~\ref{tab:ffn_step_time} and the PRNG shares. Run-to-run variation reaches $\pm 6\%$ at $W_8$ and $W_9$ and is shared across configurations. The corrections were timed in two further runs without the rounding-variance statistic. Table~\ref{tab:ffn_corr} gives their speedups, their extra time in int32 operations per upstream-gradient element, the change in MXU-category time, and their temporary memory.

\begin{table*}[t!]
\centering
\caption{Device Time per Step and PRNG Share of One SwiGLU FFN Block (Forward + Backward, per Device) on TPU v5p, from the Run behind Table~\ref{tab:ffn_step_time}. $B = 8$ ($4$ for $W_9$), $S = 2048$; $I_{\text{BF16}}$ Is the Arithmetic Intensity of the BF16 Step. Median over Repeated Steps}
\label{tab:ffn_detail}
\resizebox{\textwidth}{!}{%
\begin{tabular}{@{} l r r r r r r r r r @{}}
\toprule
 & $W_1$ & $W_2$ & $W_3$ & $W_4$ & $W_5$ & $W_6$ & $W_7$ & $W_8$ & $W_9$ \\
$(H, D_{\text{ffn}})$ & $(512, 2048)$ & $(1024, 4096)$ & $(2048, 8192)$ & $(3072, 12288)$ & $(3840, 15360)$ & $(4096, 12288)$ & $(4096, 14336)$ & $(5120, 17920)$ & $(8192, 28672)$ \\
$I_{\text{BF16}}$ (FLOP/B) & 745 & 1446 & 2731 & 3880 & 4666 & 4706 & 4823 & 5747 & 6072 \\
\midrule
\multicolumn{10}{@{}l}{\textit{Device time per step (ms)}} \\
BF16 & 0.858 & 3.141 & 12.834 & 29.582 & 44.061 & 40.502 & 44.097 & 69.861 & 90.879 \\
INT8 RTNE & 0.765 & 2.172 & 7.403 & 15.387 & 24.004 & 21.388 & 23.587 & 38.450 & 50.239 \\
Independent SR (Threefry) & 1.466 & 3.467 & 9.959 & 19.244 & 28.836 & 25.392 & 28.194 & 44.293 & 54.721 \\
S3R ($\Phi_{\text{2D}}$) & 0.798 & 2.186 & 7.412 & 15.403 & 24.025 & 21.358 & 23.549 & 38.580 & 50.173 \\
S3R ($\Phi_{\text{bcast}}$) & 0.785 & 2.149 & 7.536 & 15.404 & 24.050 & 21.390 & 23.577 & 38.588 & 50.173 \\
\midrule
\multicolumn{10}{@{}l}{\textit{PRNG share of busy time (\%)}} \\
Independent SR (Threefry) & 56.5 & 44.2 & 30.6 & 23.7 & 19.8 & 18.7 & 19.3 & 15.3 & 9.9 \\
S3R ($\Phi_{\text{2D}}$) & 0.9 & 0.3 & 0.2 & 0.0 & 0.0 & 0.0 & 0.0 & 0.0 & 0.0 \\
S3R ($\Phi_{\text{bcast}}$) & 1.3 & 0.6 & 0.2 & 0.0 & 0.0 & 0.0 & 0.0 & 0.0 & 0.0 \\
\bottomrule
\end{tabular}%
}
\vspace{-6pt}
\end{table*}

\begin{table*}[t!]
\centering
\caption{Cost of the S3R-EC Corrections on One SwiGLU FFN Block (Forward + Backward, per Device) on TPU v5p, from the Two Runs behind the Extra Times in Table~\ref{tab:ffn_step_time}, without the Rounding-Variance Statistic. Speedups: Range over the Two Runs. Extra Time: Device Time over S3R ($\Phi_{\text{2D}}$) in the Same Run, Expressed as Elementwise int32 VPU Operations per Upstream-Gradient Element, Mean of the Two Runs. MXU-Category Time: Time of the Fusions That Contain a GEMM, First Run. Temporary Memory: XLA Memory Analysis.}
\label{tab:ffn_corr}
\small
\setlength{\tabcolsep}{0pt}
\begin{tabular*}{\textwidth}{@{\extracolsep{\fill}} l r r r r r r r r r @{}}
\toprule
 & $W_1$ & $W_2$ & $W_3$ & $W_4$ & $W_5$ & $W_6$ & $W_7$ & $W_8$ & $W_9$ \\
\midrule
\multicolumn{10}{@{}l}{\textit{Speedup over BF16}} \\
INT8 RTNE & 1.12--1.14 & 1.43 & 1.67--1.71 & 1.95--1.96 & 1.84 & 1.87--1.88 & 1.87--1.92 & 1.83--1.84 & 1.82--1.83 \\
S3R ($\Phi_{\text{2D}}$) & 1.07--1.09 & 1.42 & 1.67--1.70 & 1.94--1.96 & 1.84 & 1.88 & 1.87--1.93 & 1.83--1.84 & 1.82--1.83 \\
S3R-EC, two-way & 0.95--0.96 & 1.32 & 1.68--1.70 & 1.85--1.88 & 1.78--1.80 & 1.83--1.84 & 1.82--1.86 & 1.82--1.83 & 1.75--1.76 \\
S3R-EC, dX-only & 0.94--0.95 & 1.29--1.30 & 1.65--1.66 & 1.87--1.88 & 1.78--1.79 & 1.82--1.83 & 1.82--1.86 & 1.78--1.79 & 1.78--1.79 \\
\midrule
\multicolumn{10}{@{}l}{\textit{Extra time over S3R ($\Phi_{\text{2D}}$) (int32 operations per upstream-gradient element)}} \\
S3R-EC, two-way & $+19.9$ & $+15.2$ & $-2.1$ & $+23.1$ & $+16.7$ & $+17.0$ & $+21.5$ & $+5.9$ & $+49.3$ \\
S3R-EC, dX-only & $+22.5$ & $+19.9$ & $+6.2$ & $+20.0$ & $+19.7$ & $+20.0$ & $+19.8$ & $+21.8$ & $+28.2$ \\
\midrule
\multicolumn{10}{@{}l}{\textit{Change in MXU-category time over S3R ($\Phi_{\text{2D}}$) (ms)}} \\
S3R-EC, two-way & $-0.041$ & $-0.177$ & $-0.621$ & $-0.057$ & $-0.455$ & $-0.297$ & $-0.009$ & $-0.855$ & $+0.917$ \\
S3R-EC, dX-only & $-0.051$ & $-0.084$ & $-0.417$ & $-0.175$ & $-0.063$ & $-0.074$ & $+0.024$ & $+0.075$ & $+0.157$ \\
\midrule
\multicolumn{10}{@{}l}{\textit{Temporary memory over S3R ($\Phi_{\text{2D}}$), both corrections (equal to within 1\,MiB)}} \\
MiB & $+42$ & $+69$ & $+129$ & $+230$ & $+242$ & $+194$ & $+226$ & $+282$ & $+226$ \\
\% & $+14.0$ & $+11.2$ & $+10.2$ & $+12.1$ & $+9.9$ & $+9.7$ & $+9.8$ & $+9.6$ & $+8.0$ \\
\bottomrule
\end{tabular*}
\vspace{-6pt}
\end{table*}

\textbf{Other gradient formats (one seed).} Two further regimes use one seed and the learning rates $\{3, 6\} \times 10^{-4}$. With one gradient scale per token instead of one per tensor, the scales lie on the contraction axis of the weight-gradient product, which therefore applies them in \texttt{f32}, and the SR share of $v_t$ falls to about $5 \times 10^{-4}$. With 4-bit per-tensor gradients, the SR variance makes up more than half of $v_t$. Independent SR then reaches 5.796, 0.9 behind BF16, and uncorrected S3R with $\Phi_{\text{2D}}$ reaches 6.777. Most best learning rates lie at the low end of the two-point grid, so these 4-bit losses are upper bounds. S3R-EC was not evaluated in either regime.

\textbf{Full training steps.} Table~\ref{tab:e2e_int32} repeats the full-step measurements with int32-output INT8 GEMMs, and Table~\ref{tab:e2e_profile} breaks down the device time of Llama 3.1-8B by operation category. With int32 output, S3R needs 0.9\,GB more device memory than RTNE on Llama 3.1-8B.

In the MoE models, the INT8 configurations quantize the grouped GEMMs of the routed experts, in the forward pass and in both backward products, as well as the dense projections. In Qwen3-30B-A3B, the dense projections are the attention q/k/v/o projections and the router. In DeepSeek-V2-Lite, they are the MLA down- and up-projections, the attention output projection, the shared-expert MLP, the MLP of the leading dense layer, and the router. The attention kernel, the logits, and the embedding stay in \texttt{bf16}. All MoE configurations, BF16 included, use the same grouped-GEMM backend, with the Pallas grouped-matmul kernel disabled.

The expert GEMMs operate on the token matrix sorted by expert, so S3R and the correction are adapted to it. The scale $\Delta$ is computed per shard of the grouped GEMM, and S3R uses one threshold per row of the sorted matrix ($\Phi_{\text{2D}}$). These rows have no batch and sequence structure, so the token-mean term of the correction is not applied. The two-way correction applies only the channel-mean part $\boldsymbol{\varepsilon}$ to dX and dW, and the dX-only correction applies it to dX only. Here $\boldsymbol{\varepsilon}$ is centered by its mean over the rows of the same expert group, in place of $m_b$, and normalized by the local row count. The dense projections of the MoE models receive the full corrections. No PRNG stream reaches the expert computation, so the SR key of the expert GEMMs is derived in each call by hashing a few activation rows. The dense GEMMs take their key from the trainer's PRNG stream. Because this key depends on the data, Lemma~\ref{lemma:unbiasedness} does not strictly apply to the expert GEMMs.

The grouped INT8 GEMM supports only int32 accumulation. In the \texttt{bf16}-output variant, the expert GEMMs therefore compute in int32 and round each shard's result to \texttt{bf16} before any cross-device reduction. The expert-path correction runs in two fused Pallas kernels over the expert-sorted rows. The first computes $\boldsymbol{\varepsilon}$ and its per-group sums. The second corrects the int32 dX accumulator before rounding it and, for the two-way correction, forms the per-group factor of the dW correction. Neither computes the rounding-variance statistic. Table~\ref{tab:e2e_moe_profile} breaks down the time that the two-way correction adds on Qwen3-30B-A3B; 95\% of it is spent in the two kernels.

\begin{table}[t!]
    \centering
    \caption{Full Training Step with Int32-Accumulator INT8 GEMMs: Median Step Time (ms) under the Protocol of Table~\ref{tab:e2e_step}, and Speedup over BF16 on the Dense Models ($> 1$ Is Faster)}
    \label{tab:e2e_int32}
    \resizebox{\columnwidth}{!}{%
    \begin{tabular}{@{} l r r r r c c @{}}
        \toprule
         & & & Indep.\ SR & & \multicolumn{2}{c}{Speedup over BF16} \\
        Model & BF16 & RTNE & (Threefry) & S3R & RTNE & S3R \\
        \midrule
        Llama 3.1-8B & 3862 & 4171 & 4624 & 4199 & $0.926\times$ & $0.920\times$ \\
        Qwen3-8B & 4018 & 4542 & 5006 & 4582 & $0.885\times$ & $0.877\times$ \\
        Gemma 3-12B & 6550 & 6294 & 7001 & 6322 & $1.041\times$ & $1.036\times$ \\
        \midrule
        Qwen3-30B-A3B & 6454 & 7894 & 8495 & 7888 & & \\
        DeepSeek-V2-Lite (16B) & 3659 & 4307 & 4746 & 4324 & & \\
        \bottomrule
    \end{tabular}%
    }
    \vspace{-6pt}
\end{table}

\begin{table}[t!]
    \centering
    \caption{Device Time per Step (ms per Device) of Llama 3.1-8B by Operation Category, from Execution Traces of 3 Training Steps. Self Time, Averaged over Steps and Devices; Totals Are within 0.15\% of the Step Time Measured in the Profiled Runs}
    \label{tab:e2e_profile}
    \resizebox{\columnwidth}{!}{%
    \begin{tabular}{@{} l r r r @{}}
        \toprule
        Category & BF16 & S3R (int32 acc.) & RTNE (\texttt{bf16} out) \\
        \midrule
        GEMMs (MXU) & 2236 & 1187 & 1182 \\
        Plain all-reduce (TP partial sums) & 5 & 1791 & 43 \\
        Fused all-reduce-scatter & 362 & 10 & 525 \\
        All-gather (FSDP weights, TP activations) & 484 & 349 & 319 \\
        Flash-attention kernel & 508 & 511 & 526 \\
        Elementwise and reduction fusions (VPU) & 174 & 270 & 249 \\
        Other & 88 & 79 & 75 \\
        \midrule
        Total & 3858 & 4198 & 2919 \\
        \bottomrule
    \end{tabular}%
    }
    \vspace{-6pt}
\end{table}

\begin{table}[t!]
    \centering
    \caption{Device Time per Step (ms per Device) Added by the Two-Way Correction over S3R on Qwen3-30B-A3B with \texttt{bf16}-Output INT8 GEMMs, from Execution Traces of 3 Training Steps. Kernel Calls Are per Layer (S3R $\to$ Two-Way)}
    \label{tab:e2e_moe_profile}
    \footnotesize
    \begin{tabular}{@{} l r c @{}}
        \toprule
        Component & $\Delta$ (ms) & Calls \\
        \midrule
        Expert-path correction kernels & $+325$ & \\
        \quad $\boldsymbol{\varepsilon}$ and its per-group sums & $+136$ & $0 \to 3$ \\
        \quad Correction of dX and dW factor & $+189$ & $0 \to 3$ \\
        Collectives & $+19$ & \\
        Flash-attention kernel & $+19$ & unchanged \\
        Elementwise and reduction fusions & $-33$ & \\
        Main expert GEMMs, dense GEMMs, token sort & $+4$ & unchanged \\
        Other & $+9$ & \\
        \midrule
        Total (7285 $\to$ 7628) & $+343$ & \\
        \bottomrule
    \end{tabular}
    \vspace{-6pt}
\end{table}

\begin{table}[t!]
    \centering
    \caption{S3R and S3R-EC Settings and Training Protocol}
    \label{tab:hyperparam_specs}
    \footnotesize
    \begin{tabularx}{\columnwidth}{@{} >{\raggedright\arraybackslash}p{0.27\columnwidth} >{\raggedright\arraybackslash}X @{}}
        \toprule
        \textbf{Setting} & \textbf{Value / Protocol} \\
        \midrule
        Convergence Model & 12-Layer Llama-Style Decoder: Hidden 1024, 16 Heads of Dimension 64, SwiGLU FFN Width 2816, RMSNorm, RoPE, 32{,}768-Token Vocabulary \\
        Training Dataset \& Stream & C4 \cite{raffel2020exploring} (\texttt{c4/en} 3.0.1), Packed, Each Token Seen at Most Once (Convergence); Synthetic (Full Step) \\
        Convergence Budget & 2000 Steps (32.8M Tokens), Best LR on Seed 0, 3 Seeds \\
        Sequence Length $S$ & $2048$ (FFN Block, Convergence); $4096$ (Full Step) \\
        Batch Size $B$ & $8$ ($4$ for $W_9$), FFN Block; $8$, Convergence; $4$ per Device, Full Step \\
        Peak Learning Rate $\eta_{\max}$ & Swept per Configuration (Best Values in Table~\ref{tab:convergence}); 10\% Linear Warmup, Cosine Decay to $0.1\,\eta_{\max}$ \\
        Optimizer & AdamW (Unchanged) \\
        AdamW $(\beta_1, \beta_2, \epsilon)$ & $\beta_1 = 0.9$, $\beta_2 = 0.999$, $\epsilon = 10^{-8}$ \\
        Weight Decay $\lambda$ \& Grad Clipping & $\lambda = 0.01$ (Decoupled), Global $\ell_2$-Norm Clip at $1.0$ \\
        Full-Step Numerics & \texttt{bf16} Activations; \texttt{f32} Master Weights and Optimizer States; Full Rematerialization \\
        Full-Step Timing & Median over 20 Steps after 12 Warm-Up Steps; 10th--90th Percentile Spread $\le 0.21\%$ of the Median; Repeated Runs within 0.6\% \\
        Quantized Contractions & FFN / Expert (\texttt{gate}, \texttt{up}, \texttt{down}) \& Attention Projections; Logits, Embedding in \texttt{bf16} \\
        Forward Quantization Format & Symmetric INT8 (\texttt{s8}, Range $[-127, 127]$), RTNE; Per-Token Activation and Per-Output-Channel Weight Absmax Scales \\
        Backward Quantization Format & Symmetric INT8 (\texttt{s8}) on $\widetilde{\mathbf{G}}_t$, Shared by the dX and dW GEMMs ($\alpha_{\text{share}} = 1/2$); Per-Tensor Step $\Delta = \max|\mathbf{G}'|/127$ on the Scale-Folded Gradient (per Shard for MoE Expert GEMMs) \\
        S3R Threshold Format & 8-Bit Thresholds with 7 Effective Bits and a Half-Step Offset \\
        S3R Seed Shape & $[B, S, 1]$ per Layer ($\Phi_{\text{2D}}$ Default; $\Phi_{\text{bcast}}$ Ablated) \\
        S3R-EC Correction & Two-Way: $\mathbf{E}^{\text{str}}_b$ of (\ref{eq:estr}) Removed from $\widehat{\nabla}_{\mathbf{X}}$ and $\bar{\mathbf{g}}_{\mathbf{W}}$ (Algorithm~1 in Table~\ref{tab:algorithms}); dX-Only: Removed from $\widehat{\nabla}_{\mathbf{X}}$ Only \\
        \bottomrule
    \end{tabularx}
    \vspace{-6pt}
\end{table}

\section{Sharding-Invariant Seeding in Distributed Training}
\label{app:extended_dynamics}

With partitionable Threefry, the rounding thresholds of S3R do not depend on how tensors are sharded across devices. When a run moves to a different parallelism layout, for example from pure FSDP to mixed FSDP and TP, the thresholds stay bitwise identical, and no PRNG state is exchanged between devices.

The invariance holds because every threshold is a function of global indices only, never of the device that computes it. For independent SR, this requires a counter for each of the $B S D_{\text{out}}$ elements, partitioned along the sharded tensor axes; the S3R seeds are much smaller. The default $\Phi_{\text{2D}}$ seed $\mathbf{u}_t \in [0, 1]^{B \times S \times 1}$ is a partitionable Threefry draw over the global token indices $(b, s)$ for each step and layer. The $\Phi_{\text{bcast}}$ seed $\mathbf{u}_t \in [0, 1]^{B \times 1 \times 1}$ depends only on the step, the layer index and the global batch index $b$, which are folded into the step key with \texttt{jax.random.fold\_in}. Neither seed has a hidden dimension, so both are invariant to TP sharding of the hidden dimension; $\Phi_{\text{bcast}}$ is also invariant to sequence sharding and to token-level dispatch to MoE experts.

We verified the invariance on TPU by generating the noise and threshold tensors of a $[8, 256, 2048]$ upstream gradient under three FSDP$\times$TP factorizations of the same device set (pure FSDP, mixed, and pure TP) and comparing them bitwise with an unsharded generation. With partitionable Threefry, the JAX default used by all S3R configurations, independent noise, $\Phi_{\text{2D}}$ thresholds and $\Phi_{\text{bcast}}$ thresholds are bitwise identical under all three factorizations, and the compiler adds no collectives. Independent noise is thus invariant as well; S3R keeps the property with a far smaller seed. Non-partitionable Threefry is also bitwise identical but adds all-to-all and collective-permute operations.

A separate CPU check with emulated devices in four mesh shapes agrees. With partitionable Threefry, independent noise, $\Phi_{\text{2D}}$ and $\Phi_{\text{bcast}}$ (in both the Algorithm~1 form and the fold-in form) give bitwise-identical thresholds and INT8 outputs in all 40 comparisons, whereas naive per-device local draws differ in 87--99\% of entries. The emulation does not cover the collectives inserted by the compiler.

\fi
\end{document}